\documentclass[journal]{IEEEtran}
\usepackage[dvips]{graphicx}
\usepackage{epsfig,multirow}
\usepackage{amsthm,amsmath,amssymb,mathrsfs,bm}
\usepackage{amsfonts,dsfont,color,bbm}
\usepackage{algorithm,algorithmic}
\usepackage[font=footnotesize,labelfont=bf]{caption}
\usepackage{subcaption,slashbox}
\usepackage{cite,epstopdf,epsf,epsfig}
\usepackage{resizegather}
\usepackage{booktabs}
\usepackage{float}
\usepackage{mathtools}
\usepackage{xr}
\usepackage{dsfont}
\newcommand{\nj}{\hat{n}_j(t)}
\newcommand{\njm}{\hat{n}_j(t-1)}
\newcommand{\nlm}{\hat{n}_l(t-1)}

\newcommand{\nst}{\hat{n}_{\st(t)}}
\newcommand{\nstm}{\hat{n}_{\st(t-1)}}
\newcommand{\nstau}{\hat{n}_{\st(\tau)}}

\newcommand{\Nstm}{\hat{N}_{\st(1:t-1)}}

\newcommand{\xst}{X_{\sT(t)}}
\newcommand{\wfr}{\frac{w_i(t)}{W_t}}
\newcommand{\yh}{\hat{y}_i(t)}
\newcommand{\xh}{\hat{x}_j(t)}

\newcommand{\vst}{\frac{c}{p_{\sT(t)} \sqrt{KT}}}
\newcommand{\summb}{\sum_{\sT \in \mb}}

\newcommand{\vh}{\hat{u}_i(t)}
\newcommand{\ub}{\frac{c \hat{u}_i(t)}{\sqrt{KT}}}
\newcommand{\yt}{\tilde{y}_i(t)}
\newcommand{\yrt}{\tilde{y}_r(t)}

\newcommand{\lnratioT}{\ln \Big( \frac{W_T}{W_1} \Big)} 
  
\newcommand{\sumrA}{\sum_{r \in A^*}}   
   
\newcommand{\sumT}{\sum_{t=1}^{T}}   
\newcommand{\sumTh}{\sum_{j \in [K]-U_0(t)}}
    
\newcommand{\sumK}{\sum_{j=1}^K}   
          
\newcommand{\sumNr}{\sum_{i=1}^{N_r}}   
\newcommand{\sumNrL}{\sum\limits_{i=1}^{N_r}}   
\newcommand{\sumsrt}{\sum\limits_{\st(t)=r}}
\newcommand{\sumst}{\sum\limits_{\st}}
\newcommand{\sumstm}{\sum\limits_{\st(1:t-1)}}
\newcommand{\sumsjtm}{\sum\limits_{\st(t-1)=j}}

\newcommand{\Sw}{S} 
\newcommand{\1}{\mathds{1}} 
\newcommand{\mb}{\textbf{M}^*_T} 
\newcommand{\lns}{\ln (\frac{eK(T-1)}{\Sw -1})} 
\newcommand{\lnsd}{\ln \! \Big(\! \frac{eK(T-1)}{(\Sw -1)\delta} \! \Big)}

\newcommand{\st}{\textbf{s}_t} 
\newcommand{\sT}{\textbf{s}_T} 
\newcommand{\MT}{\textbf{M}_T} 
\newcommand{\wst}{w_{\textbf{s}_t}} 
\newcommand{\wstm}{w_{\st(1:t-1)}} 

\newcommand{\pst}{\pi_{\textbf{s}_t}} 
\newcommand{\psst}{\pi(\textbf{s}_t | \textbf{s}_t(t-1))} 
\newcommand{\pstm}{\pi_{\st(1:t-1)}}

\newcommand{\xhst}{\hat{x}_{\textbf{s}_t(t)}} 
\newcommand{\etav}{\frac{m \gamma}{2K}} 
\newcommand{\gammavms}{\sqrt{\frac{K \lns}{mT}}} 
\newcommand{\gammavmp}{\sqrt{\frac{K \ln \frac{N_r}{m}}{mT}}} 
\newcommand{\cvms}{\sqrt{m \Sw \lnsd}}
\newcommand{\cvmp}{\sqrt{m \ln \frac{N_r}{\delta}}}
\newcommand{\betav}{\frac{\Sw -1}{T-1}}
\newcommand{\betave}{\frac{\Sw -1}{e(T-1)}}
\newcommand{\yit}{Y_i(t)}

\DeclareMathOperator*{\argmax}{arg\,max}

\newtheorem{theorem}{Theorem}[section]
\newtheorem{corollary}{Corollary}[section]
\newtheorem{lemma}[theorem]{Lemma}
\newtheorem{remark}{Remark}[section]

\usepackage{balance}

\newcommand{\define}{\overset{\triangle}{=}}

\begin{document}
\renewcommand{\thepage}{}

\title{Minimax Optimal Algorithms for Adversarial Bandit Problem with Multiple Plays}

\author{N. Mert Vural, Hakan Gokcesu, Kaan Gokcesu and Suleyman S. Kozat, \textit{Senior Member, IEEE}
\thanks{This work is supported in part by Turkish Academy of Sciences Outstanding Researcher Programme.}
\thanks{ N. M. Vural, H. Gokcesu and S. S. Kozat are with the Department of Electrical and Electronics
Engineering, Bilkent University, Ankara 06800, Turkey, e-mail: vural@ee.bilkent.edu.tr, 
hgokcesu@ee.bilkent.edu.tr, kozat@ee.bilkent.edu.tr.}
\thanks{K. Gokcesu is with the Department of Electrical Engineering and Computer
Science, Massachusetts Institute of Technology, Cambridge, MA 02139 USA,
e-mail: gokcesu@mit.edu.}
}


\maketitle
\begin{abstract}
We investigate the adversarial bandit problem with multiple
plays under semi-bandit feedback. We introduce a highly efficient
algorithm that asymptotically achieves the performance of the
best switching $m$-arm strategy with minimax optimal regret bounds. 
To construct our algorithm, we introduce a new expert advice algorithm for the multiple-play setting.
By using our expert advice algorithm, we
additionally improve the best-known high-probability bound for the multi-play setting by $O(\sqrt{m})$.
Our results are guaranteed to hold in an individual sequence
manner since we have no statistical assumption on the bandit arm
gains.  Through an extensive set of experiments involving
synthetic and real data, we demonstrate significant performance
gains achieved by the proposed algorithm with respect to the
state-of-the-art algorithms.
\end{abstract}

\section{Introduction}\label{sec:intro}
\subsection{Preliminaries}
Multi-armed bandit problem is extensively investigated in the  online learning\cite{CBianchi2006,Chang2010,Kaan2018,Tekin2015,Auer1995,Tekin2011} and signal processing \cite{Liu2010,Wang2012,Gai2014,Vakili2013,Liu2} literatures, especially for the applications where feedback is limited, and exploration-exploitation
must be balanced optimally. In the classical framework, the multi-armed bandit problem deals with
choosing a single arm out of $K$ arms at each round so as to maximize the total reward. We study the multiple-play version of this problem, where we choose an $m$ sized subset of $K$ arms at each round. We assume that
\begin{itemize}
\item The size $m$ is
constant throughout the game and known a priori by the learner.
\item The order of arm selections does not have an effect on the arm gains.
\item The total gain of the selected $m$ arms is the sum of the gains of the selected individual arms. 
\item We can observe the gain of \emph{each one of the selected $m$ arms} at the end of each round. Since we can observe
the gains of the individual arms in the selected subset, we also  obtain partial information about the other possible subset selections with common individual arms (semi-bandit feedback).
\end{itemize}
We point out that this framework is extensively used to model several real-life problems such as online shortest path and online advertisement placement\cite{KoolOSP,Nakamura2005}.

We investigate the multi-armed bandit problem with multiple plays (henceforth \emph{the MAB-MP problem}) in an individual sequence framework where we make no statistical assumptions on the data in order to model chaotic, non-stationary or even adversarial environments\cite{Auer1995}. To this end, we evaluate our algorithms from a competitive perspective and define our performance with respect to a competing class of strategies. As the competition class, we use the switching $m$-arm strategies, where the term $m$-arm is used to denote any distinct $m$ arms. We define the class of the switching $m$-arm strategies as the set of all deterministic $m$-arm selection sequences, where there are a total of $\binom{K}{m}^T$ sequences in a $T$ length game. We evaluate our performance with respect to the best strategy (maximum gain) in this class. We note that similar competing classes are widely used in the control theory\cite{18,19}, neural networks\cite{22,23}, universal source coding theory\cite{24,25,26} and computational learning theory\cite{20,21,Vovk1997}, due to their modelling power to construct competitive algorithms that also work under practical conditions.

In the class of the switching $m$-arm strategies, the optimal strategy is, by definition, the one whose $m$-arm selection yields the maximum gain at each round of the game. If the optimal strategy changes its $m$-arm selection $S-1$ times, i.e., $S-1$ switches, we say the optimal strategy has $S$ segments. Each such segment constitutes a part of the game (with possibly different lengths) where the optimum $m$-arm selection stays the same. For this setting, which we will refer as \emph{the tracking the best $m$-arm setting}, we introduce a highly efficient algorithm that asymptotically achieves the performance of the best switching $m$-arm strategy with minimax optimal regret bounds. 

To construct our tracking algorithm, we follow the derandomizing
approach \cite{Vovk1997}. We  consider each $m$-arm strategy as an expert with a predetermined $m$-arm selection sequence, where the number of experts grows with $\binom{K}{m}^T$, and combine them in an expert advice algorithm under semi-bandit feedback. Although we have exponentially many experts, we derive an optimal regret bound with respect to the best $m$-arm strategy with a specific choice of initial weights. We then efficiently implement this algorithm with a weight-sharing network, which requires $O(K \log K)$ time and $O(K)$ space. We note that our algorithm requires prior knowledge of the number of segments in the optimal strategy, i.e., $S$. However, it can be extended to a truly online form, i.e., without any knowledge $S$, by using the analysis in \cite{Kaan2018} with an additional $O(\log T)$ time complexity cost.

We point out that the state-of-the-art expert advice algorithms\cite{Kale2010,Auer1995} cannot combine $m$-arm sequences optimally due to the additional $O(\sqrt{m})$ term in their regret bounds. Therefore, to construct an optimal algorithm, we introduce an optimal expert advice algorithm for the MAB-MP setting. In our expert advice algorithm, we utilize the structure of the expert set in order to improve the regret bounds of the existing expert advice algorithms\cite{Kale2010,Auer1995} up to $O(\sqrt{m})$. We then combine $m$-arm sequences optimally in this algorithm and obtain the minimax optimal regret bound. By using our expert advice algorithm, we additionally improve the best-known high-probability bound\cite{Neu2016} by $O(\sqrt{m})$, hence, close the gap between high-probability bounds\cite{Neu2016} and the expected regret bounds\cite{Kale2010, Uchiya2010}. In the end, we also demonstrate significant performance gains achieved by our algorithms with respect to the state-of-the-art algorithms \cite{Gyorgy2007,Kale2010,Uchiya2010,Neu2016,Auer1995} through an extensive set of experiments involving synthetic and real data.
\subsection{Prior Art and Comparison}
The MAB-MP problem is mainly studied under three types of feedback: The full-information\cite{Takimoto2003}, where the gains of all arms are revealed to the learner, the semi-bandit feedback\cite{Audibert2012,Gyorgy2007,Uchiya2010,Kale2010,Neu2016}, where the gains of the selected $m$ arms are revealed, and the full bandit feedback\cite{Dani,Ab2008,CB2012}, where only the total gain of the selected $m$-arm is revealed. Since our study lies in the semi-bandit scenario, we focus on the relevant studies for the comparison.

The adversarial MAB-MP problem where the player competes against the best fixed $m$-arm under semi-bandit feedback has a regret lower bound of $O(\sqrt{mKT})$\footnote{We use big-$O$ notation, i.e., $O(f(x))$, to ignore constant factors and use soft-$O$ notation, i.e., $\tilde{O}(f(x))$, to ignore the logarithmic factors as well.} for $K$ arms in a $T$ round game \cite{Audibert2012}. On the other hand, a direct application of \textit{Exp3} \cite{Auer1995}, i.e., the state-of-the-art for $m=1$,  achieves a regret bound $O(m^{3/2}\sqrt{K^mT\ln K})$ with $O(K^m)$ time and space complexity. One of the earliest studies to close this performance gap with an efficient algorithm is by Gy{\"o}rgi et al.\ \cite{Gyorgy2007}. They derived a regret bound $O(m^{3/2}\sqrt{KT\ln K} )$ with respect to the best fixed $m$-arm in hindsight with $O(K)$ time complexity. This result is improved by Kale et al.\ \cite{Kale2010} and Uchiya et al.\ \cite{Uchiya2010} whose algorithms guarantee a regret bound $O(\sqrt{mKT \ln (K/m)})$ with $O(K^2)$  and $O(K \log K)$ time complexities respectively. Later, Audibert et al.\ \cite{Audibert2012} achieved the minimax optimal regret bound by the \textit{Online Stochastic Mirror Descent (OSMD)} algorithm. The efficient implementation of \textit{OSMD} is studied by Suehiro et al. \cite{Suehiro2012} whose algorithm has $O(K^6)$ time complexity.

We emphasize that although minimax optimal bound has been achieved, all of these results have been proven to hold only in expectation. In practical applications, these algorithms suffer from the large variance of the unbiased estimator, which leads $O(T^{3/4})$ regret in the worst case\cite{Auer1995,Beyg2011}. This problem is addressed by Gy{\"o}rgi et al.\ \cite{Gyorgy2007} and Neu et al.\ \cite{Neu2016} for the MAB-MP problem. They respectively derived $O(m^{3/2} \sqrt{KT \log (K/\delta)})$ and $O(m  \sqrt{KT \log (K/m \delta)})$ regret bounds holding with probability $1-\delta$.

In this paper, we introduce algorithms that achieve \emph{minimax optimal regret} (up to logarithmic terms) with high probability for both the vanilla MAB-MP and the tracking the best $m$-arm settings. In order to generalize both settings in an optimal manner, we first introduce an optimal expert-mixture algorithm for the MAB-MP problem in Section III. In our expert-mixture algorithm, differing from the state-of-the-art\cite{Kale2010}, we exploit the structure of the expert set, and introduce the notion of underlying experts. By exploiting the structure of the expert set, we improve the regret bound of the state-of-the-art expert mixture algorithm for the MAB-MP setting \cite{Kale2010} up to $O(\sqrt{m})$ and obtain the optimal regret bound against to the best expert, which can follow \emph{any arbitrary strategy}. We then consider the set of the deterministic $m$-arm in our expert-mixture algorithm in Remark III.1 and close the gap between high-probability bounds\cite{Neu2016,Gyorgy2007} and the expected regret bounds\cite{Kale2010,Uchiya2010} for the vanilla MAB-MP setting.

In addition to our improvement in high-probability bound, we use our optimal expert mixture algorithm to develop a tracking algorithm for the MAB-MP setting. We note that when competing against the best switching $m$-arm strategy (as opposed to the best-fixed $m$-arm), the minimax lower bound can be derived as $O(mSKT)$\footnote{When competing against the best switching bandit arm strategy (as opposed to the best fixed arm strategy), we can apply $O(\sqrt{mKT})$ bound separately to each one of $S$ segment (if we know the switching instants). Hence, maximization of the total regret bound yields a minimax bound of $O(\sqrt{mSKT})$ since the square-root function is concave and the bound is maximum when each segment is of equal length $T/S$.}. However, similar to the case of \textit{Exp3}, the direct implementation of the traditional multi-armed bandit algorithms into this problem suffers poor performance guarantees. To the best of our knowledge, only Gy{\"o}rgy et al. \cite{Gyorgy2007} studied competing against the switching $m$-arm sequences and derived $\tilde{O}(m^{3/2}\sqrt{SKT})$ regret bound holding with probability $1-\delta$. In Section IV, by mixing the sets of switching $m$-arm sequences optimally in our expert mixture algorithm, we improve this result to $\tilde{O}(\sqrt{mSKT})$ regret bound holding with probability $1-\frac{S-1}{e(T-1)}\delta$. We note that the computational complexity of our final algorithm is $O(K \log K)$, whereas Gy{\"o}rgy et al.'s algorithm \cite[Section 6]{Gyorgy2007} requires $O(\min(KT,\binom{K}{m}))$ per round. Therefore, we also provide a highly efficient counterpart of the state-of-the-art.
\subsection{Contributions}
Our main contributions are as follows:
\begin{itemize}
\item As the first time in the 
literature, we introduce an online algorithm, i.e., \textit{Exp3.MSP}, that truly achieves (with
minimax optimal regret bounds) the performance of the best multiple-arm selection strategy. 
\item We achieve this performance with computational complexity only 
log-linear in the arm number, which is significantly smaller than the computational complexity of the state-of-the-art\cite{Gyorgy2007}.
\item In order to obtain the minimax optimal regret bound with \textit{Exp3.MSP}, we introduce an optimal expert mixture algorithm for the MAB-MP setting, i.e., \textit{Exp4.MP}. We derive a lower bound for the MAB-MP with expert advice setting and mathematically show the optimality of the \textit{Exp4.MP} algorithm.
\item  By using \textit{Exp4.MP}, we additionally improve the best-known high-probability bound for the multiple-play setting by $O(\sqrt{m})$, hence, close the gap between high-probability bounds\cite{Neu2016,Gyorgy2007} and the expected regret bounds\cite{Kale2010,Uchiya2010}.
\end{itemize}

\subsection{Organization of the Paper}
The organization of this paper is as follows: In Section \ref{sec:prob} we formally define
the adversarial multi-armed bandit problem with multiple plays. In Section \ref{sec:exp4m}, we introduce an optimal expert mixture algorithm for the MAB-MP setting. In Section \ref{sec:exp3ms}, by using our expert mixture algorithm , we construct an algorithm that competes with the best switching $m$-arm strategy in a computationally efficient way. In Section \ref{sec:experiments}, we demonstrate the performance of our algorithms via an extensive set of experiments. We conclude with final remarks in Section \ref{sec:conclusion}. 
\section{Problem Description}\label{sec:prob}
We use bracket notation $[n]$ to denote the set of the first $n$ positive integers, i.e., $[n]=\{1,\cdots,n\}$. We use $\textbf{C}([n],m)$ to denote the $m$-sized combinations of the set $[n]$. We use $[K]$ to denote the set of arms and $\textbf{C}([K],m)$ to denote the set of all possible $m$-arm selections.  We use $\textbf{1}_A$ to denote the column vector, whose $j^{th}$ component is $1$ if $j \in A$, and $0$ otherwise.

We study the MAB-MP problem, where we have $K$ arms, and randomly select an $m$-arm at each round $t$. Based on our $m$-arm selection $U(t) \in \textbf{C}([K],m)$, we observe only the gain of the selected arms, i.e., $x_i(t) \in [0,1]$ for $ i \in U(t)$, and receive their sum as the gain of our selection $U(t)$. We assume $x_i(t) \in [0,1]$ for notational simplicity; however, our derivations hold for any bounded gain after shifting and scaling in magnitude. We work in the adversarial bandit setting such that we do not assume any statistical model for the arm gains $x_i(t)$. The output $U(t)$ of our algorithm at each round $t$ is strictly online and randomized. It is a function of only the past selections and observed gains.

In a $T$ round game, we define the variable $\textbf{M}_T$, which represents a deterministic $m$-arm selection sequence of length $T$, i.e., $\textbf{M}_T(t) \in \textbf{C}([K],m)$ for $t=1,\cdots,T$. In the rest of the paper, we refer to each such deterministic $m$-arm selection sequence, $\textbf{M}_T$, as an $m$-arm strategy. The total gain of an $m$-arm strategy and the total gain of our algorithm (for this section, say the name of our algorithm is \textit{ALG}) are respectively defined as
\begin{equation*}
G_{\textbf{M}_T} \triangleq \sum_{t=1}^T \sum_{i \in \textbf{M}_T(t) } x_i(t), \textrm{ and } G_{ALG} \triangleq \sum_{t=1}^T \sum_{i \in U(t) } x_i(t).
\end{equation*}
Since we assume no statistical assumptions on the gain sequence, we define our performance with respect to the optimum strategy $\textbf{M}^*_T$, which is given as $\textbf{M}^*_T= \argmax_{\textbf{M}_T} G_{\textbf{M}_T}$. In order to measure the performance of our algorithm, we use the notion of \textit{regret} such that
\begin{equation*}
R(T)\triangleq  G_{\textbf{M}^*_T}- G_{ALG}.
\end{equation*} 
There are two different regret definitions for the randomized algorithms: the expected regret and high-probability regret. Since the algorithms that guarantee high-probability regret yield more reliable performance \cite{Auer1995,Beyg2011}, we provide high-probability regret with our algorithms. High-probability regret is defined as
\begin{equation*}
\textbf{Pr} \Big[ R(T) \geq \epsilon \Big] \leq \delta,
\end{equation*} 
which means that the total gain of our selections up to $T$ is not much smaller than the total gain of the best strategy $\textbf{M}^*_T$ with probability at least $1-\delta$. 

The regret $R(T)$ depends on how hard it is to learn the optimum $m$-arm strategy $\textbf{M}^*_T$. Since at every switch  we need to learn the optimal $m$-arm from scratch, we quantify the hardness of learning the optimum strategy by the number of segments it has. We define the number of segments as $S= 1+ \sum_{t=2}^T \1_{\mb(t-1) \not = \mb(t)}$. Our goal is to achieve that minimax optimal regret up to logarithmic factors with high probability, i.e.,
$$\textbf{Pr} \Big[ R(T) \geq \tilde{O}(\sqrt{mSKT}) \Big] \leq \delta.$$
\section{MAB-MP with Expert Advice} \label{sec:exp4m}
\begin{algorithm}[t!]
\algsetup{linenosize=\small}
\small
	\caption{Exp4.MP}\label{alg:algexp4m}
	\begin{algorithmic}[1]
		\STATE \textbf{Parameters:} $\eta, \gamma \in [0,1]$ and $c \in R^+$
		\STATE \textbf{Initialization:} $w_i(1) \in R^+$ for $i \in [N_r]$
		\FOR{$t=1$ \TO $T$}
		\STATE Get the actual advice vectors $\bm{\xi}^1(t),\cdots, \bm{\xi}^{N}(t)$ \label{alg1:act}
		\STATE Find the underlying experts $\bm{\zeta}^1(t),\cdots,\bm{\zeta}^{N_r}(t)$ \label{alg1:und}
		\STATE $v_j(t)= \sumNrL \frac{w_i(t)\zeta^i_j(t)}{\sum_{l=1}^{N_r} w_l(t) }$ for $j\in [K]$ \label{alg1:arm}
		\IF{$\argmax_{j \in [K]} v_j(t) \geq  \frac{(1/m) -(\gamma/K)}{(1-\gamma)}$} \label{alg1:cap_init}
		\STATE  {Decide $\alpha_t$ as
		$\frac{\alpha_t}{\sum\limits_{v_j(t) \geq \alpha_t} \alpha_t + \sum\limits_{v_j(t) < \alpha_t} v_j(t) } = \frac{(1/m) -(\gamma/K)}{(1-\gamma)}$} \label{alg1:alpha_t}
		\STATE Set $U_0(t)= \lbrace	 j: v_j(t) \geq \alpha_t \rbrace$ \label{alg1:u0}
		\STATE $v'_j(t)=\alpha_t$ for $j \in U_0(t)$ \label{alg1:u01}
		\ELSE
		\STATE Set $U_0(t) = \emptyset$
		\ENDIF 
		\STATE Set $v'_j(t) = v_j(t)$ for $j \in [K]- U_0(t)$ \label{alg1:cap_fin}
		\STATE $p_j(t)= m\Big((1-\gamma) \frac{v'_j(t)}{\sum\limits_{l=1}^{K} v'_l(t)} + \frac{\gamma}{K}\Big) $ for $j \in [K]$ \label{alg1:arm_prob} 
		\STATE Set $U(t)=$ DepRound$(m,(p_1(t),\cdots,p_K(t)))$ \label{alg1:depround}
		\STATE Observe and receive $x_j(t) \in [0,1]$ for each $j \in U(t)$ 
		\STATE $\hat{x}_j(t) = x_j(t)/p_j(t)$  for $j \in U(t)$ \label{alg1:arm_est1}
		\STATE $\hat{x}_j(t) = 0$   for $j \in [K] - U(t)$ \label{alg1:arm_est2}
		\FOR{$i=1$ \TO $N_r$}
		\STATE $\hat{y}_i(t)= \sumTh \zeta^i_j(t) \hat{x}_j(t)$ \label{alg1:und_est}
		\STATE $\hat{u}_i(t) = \sumTh \zeta^i_j(t)/p_j(t)$ \label{alg1:und_up}
		\STATE $w_i(t+1)=w_i(t) \exp\Big(\eta ( \hat{y}_i(t) + \frac{c }{\sqrt{KT}} \hat{u}_i(t))   \Big)$ \label{alg1:upd}
		\ENDFOR
		\ENDFOR
	\end{algorithmic}
\end{algorithm}
\begin{figure}[t!]
    \centering
     \includegraphics[width=0.35\textwidth]{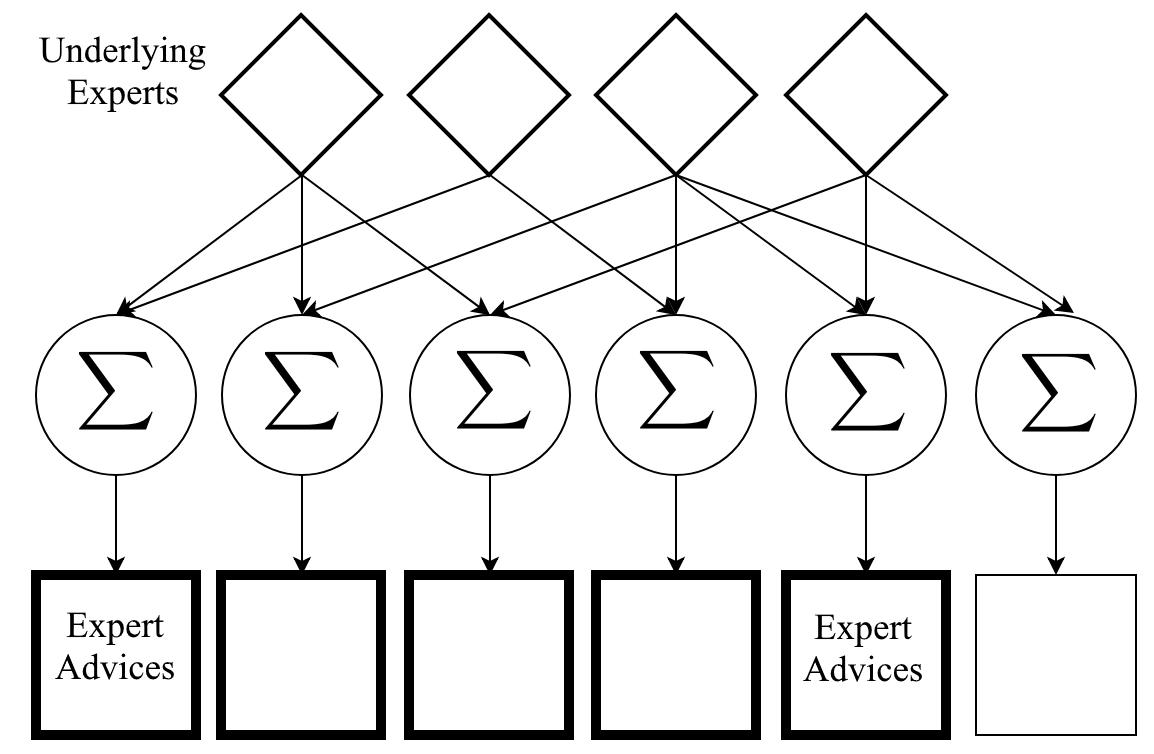}\\
     \caption{In \textit{Exp4.MP}, instead of directly using the expert set, we use 
	an underlying expert set, whose sum of $m$-combinations constitute a set containing the expert 
	advices. In the figure, the diamonds represent the underlying experts. The squares represent the sum of 
	$m$-	combinations. The bold squares are the expert advices presented to the algorithm. We note that 
    in this figure, $m=2$, $N_r=4$, $N=5$.}\label{fig:us}
\end{figure}
In this section, we consider selecting an $m$-arm with expert advice and introduce an optimal expert-mixture algorithm for the MAB-MP setting. We note that the primary aim of this section is to provide an optimal expert advice framework for the MAB-MP setting, on which we develop our optimal tracking algorithm in Section IV. By using our expert mixture algorithm, we additionally improve the best-known high-probability regret bound for the MAB-MP setting by $O(\sqrt{m})$ in the last remark of this section.

For this section, we define the phrase "expert advice" as the reference policies (or vectors) of the algorithm. The setting is as follows: At each round, each expert presents its $m$-arm selection advice as a $K$-dimensional vector, whose entries represent the marginal probabilities for the individual arms. The algorithm uses those vectors, along with the past performance of the experts, to choose an $m$-arm. The goal is to asymptotically achieve the performance of the best expert with high probability. For this setting, we introduce an optimal algorithm \textit{Exp4.MP}, which is shown in Algorithm \ref{alg:algexp4m}. In \textit{Exp4.MP}, instead of directly using the expert set, we use an \textit{underlying expert set} to utilize the possible structure of the expert set. An underlying expert set is defined as a non-negative vector set, whose sum of $m$-combinations constitute a set containing the expert advices (see Figure \ref{fig:us}). By using an underlying expert set, we replace the dependence of the regret on the size of the expert set $N$  with the size of the underlying expert set $N_r$, thus, obtain the minimax lower bound in the soft-Oh sense (proven in the following). In the rest of the paper, we use the term  \emph{underlying experts} to denote the elements of the underlying expert set, and the term \emph{actual experts} (respectively \textit{actual advice vectors}) to denote the experts (respectively expert advices) presented to the algorithm.

In \textit{Exp4.MP}, we first get the actual advice vectors $\bm{\xi}^k(t)$  for $k \in [N]$ in line \ref{alg1:act}. Since the entries of the actual advice vectors represent the marginal probabilities for the individual arms, they satisfy
\begin{align*}
\sum_{j=1}^K \xi^k_j(t)=m, \quad \max_{1 \leq j \leq K} \xi^k_j(t) \leq 1, \quad &\min_{1 \leq j \leq K}  \xi^k_j(t) \geq 0.
\end{align*}
Then we find the underlying experts, i.e., $\bm{\zeta}^i(t)$ for $i \in [N_r]$, in line \ref{alg1:und}. We note that for the algorithms presented in this paper, we derive the underlying expert sets a priori. Therefore, our algorithms do not explicitly compute the underlying experts at each round.

In the algorithm, we keep a weight for each underlying expert, i.e., $w_i(t)$ for $i \in [N_r]$. We use those weights as confidence measure to find the arm weights, i.e., $v_j(t)$, in line \ref{alg1:arm} as follows:
\begin{equation}
v_j(t)= \sumNrL \frac{w_i(t)\zeta^i_j(t)}{\sum_{l=1}^{N_r} w_l(t) } \textrm{ for }j\in [K].
\end{equation}

In order to select an $m$-arm, the expected total number of selection should be $m$, i.e., {\small$\sum_{j=1}^K p_j(t)= m$}. To satisfy this, we cap the arm weights so that the arm probabilities are kept in the range $[0,1]$. For the arm capping, we first check if there is an arm weight larger than {\small$\frac{(1/m) -(\gamma/K)}{(1-\gamma)}$} in line \ref{alg1:cap_init}. If there is, we find the threshold $\alpha_t$, and define the set $U_0(t)$ that includes the indices of the weights larger than $\alpha_t$, i.e., $U_0(t)= \lbrace	 j: v_j(t) \geq \alpha_t \rbrace$. We set the temporal weights of the arms in $U_0(t)$ to $\alpha_t$, i.e.,  $v'_j(t)= \alpha_t$ for $j \in U_0(t)$, and leave the other weights unchanged, i.e., $v'_j(t)= v_j(t)$ for $j \in [K]-U_0(t)$ (The implementation of this procedure is detailed in Appendix \ref{appa}). We then calculate the arm probabilities with the capped arm weights by
\begin{equation}
p_j(t)= m\Big((1-\gamma) \frac{v'_j(t)}{\sum\limits_{l=1}^{K} v'_l(t)} + \frac{\gamma}{K}\Big)\textrm{ for }j \in [K].
\end{equation}
In order to efficiently select  $m$  distinct arms with the marginal probabilities $p_j(t)$, we employ Dependent Rounding (DepRound) algorithm\cite{Gandhi2006} in line \ref{alg1:depround} (For the description of DepRound, see Appendix \ref{appa}). After selecting an $m$-arm, we observe the gain of each one of the selected $m$ arms and receive their sum as the reward of the round. 

To update the weights of the underlying experts, i.e., $w_i(t)$ for $i \in [N_r]$, we first find the estimated arm gains in lines \ref{alg1:arm_est1}-\ref{alg1:arm_est2}:
\begin{align}
\label{priorweight}
\hat{x}_j(t)=  \left\{ \arraycolsep=1.1pt\def\arraystretch{1.3}
\begin{array}{ll}
\frac{x_j(t)}{p_j(t)} \quad & \textrm{if} \: j \in U(t) \\
0 & \textrm{otherwise.} 
\end{array} \right.
\end{align}
Then by using the estimated arm gains $\hat{x}_j(t)$ for $j \in [K]- U_0(t)$, we calculate the estimated expected gain of the underlying experts by
\begin{equation}
\hat{y}_i(t)= \! \! \sumTh \! \zeta^i_j(t) \hat{x}_j(t)  \textrm{ for }i\in [N_r].
\end{equation}
In order to obtain high-probability bound, we use upper confidence bounds. However, we note that we cannot directly use the upper confidence bound of the single arm setting \cite{Beyg2011} since \textit{Exp4.MP} includes an additional non-linear weight capping in lines \ref{alg1:cap_init}-\ref{alg1:cap_fin}. In the following, we show that we can use a similar upper bounding technique by not including the capped arm weights $U_0(t)$, i.e.,
\begin{equation}
\hat{u}_i(t) = \sumTh  \frac{\zeta^i_j(t)}{p_j(t)}.
\end{equation}
Then by using $\hat{y}_i(t)$ and $\hat{u}_i(t)$, we update the weights $w_i(t)$ in line \ref{alg1:upd} by
\begin{equation}
w_i(t+1)=w_i(t) \exp\Big(\eta ( \hat{y}_i(t) +  \frac{ c }{\sqrt{KT}} \hat{u}_i(t) )  \Big),
\end{equation}
where $\eta$ is the learning rate and $c/\sqrt{KT}$ is the scaling factor, which determines the range of the confidence bound. 

For the following theorems, we respectively define the total gain of the underlying expert with index $i$, and its estimation as
\begin{equation}
G_i \triangleq \sum_{t=1}^T \bm{\zeta}^i(t) \cdot \bm{x}(t) \textrm{ and } \hat{G}_i\triangleq \sum_{t=1}^T \bm{\zeta}^i(t) \cdot \bm{\hat{x}}(t),
\end{equation}
where $\bm{x}(t)$ and $\bm{\hat{x}}(t)$ are the column vectors containing the real and the estimated arm gains, i.e., $\bm{x}(t)=[x_1(t),\cdots,x_K(t)]^T$ and $\bm{\hat{x}}(t)=[\hat{x}_1(t),\cdots,\hat{x}_K(t)]^T$. Let us define a set $A$ that includes $m$ arbitrary underlying experts, i.e., $A \in \textbf{C}([N_r],m)$. Then, by using the total gain of the underlying experts in the best $A$ (in terms of the total gain), the total gain of the best actual expert can be written as
\begin{equation}
G_{max} = \max_{A \in \textbf{C}([N_r],m)}  \sum_{i \in A} G_i.
\end{equation}
We also define  the upper bounded estimated gain of a set $A$, i.e $\hat{\Gamma}_A$, and the set with the maximum upper bounded estimated gain, i.e., $A^*$, as follows:
{\small
\begin{align}
\label{astar}
&\hat{\Gamma}_A \triangleq \sum_{i \in A} \hat{G}_i+ \frac{c}{\sqrt{KT}}  \sum_{t=1}^T  \sum_{i \in A}  \hat{u}_i(t) \textrm{ and }
A^*=\argmax_{A \in \textbf{C}([N_r],m)} \hat{\Gamma}_A.
\end{align}
}In the following theorem, we provide a useful inequality that relates $\hat{\Gamma}_{A^*}$, $G_{Exp4.MP}$ and the initial weights of the underlying experts in $A^*$, i.e., $w_i(1)$ for $i \in A^*$ under a certain assumption. This inequality will be used to derive regret bounds for our algorithms in Corollary \ref{probcor} and Theorem \ref{switchcor}, where we ensure that the assumption in Theorem \ref{the4} holds.
\begin{theorem}
\label{the4}
Let $W_1$  denote $\sum_{i=1}^{N_r} w_i(1)$. Assuming 
\begin{equation*}
\eta (\yh + \ub) \leq 1,  \quad \forall i \in [N_r] \textrm{ and } \forall t \in [T]
\end{equation*}
\textit{Exp4.MP} ensures that
\begin{align}
(1- \gamma - 2 \eta \frac{K}{m} ) \hat{\Gamma}_{A^*} + &\frac{(1-\gamma)}{\eta} \Big( \sum_{i \in A^*} \ln(w_i(1)) - m \ln \frac{W_1}{m} \Big) \!  \nonumber \\
&\leq \! G_{Exp4.MP} + c \sqrt{KT} +   \frac{ \eta  c^2 2 K }{\gamma m}  \label{regretformula}
\end{align}
holds for any $K,T > 0$.
\end{theorem}
\begin{proof}See Appendix B.
\end{proof}
In the following corollary, we derive the regret bound of \textit{Exp4.MP} with uniform initialization.
\begin{corollary}
\label{probcor}
If Exp4.MP is initialized with $w_i(1)=1$ $\forall i \in [N_r]$, and run with the parameters
\begin{equation*}
\eta=\etav \quad \gamma=\gammavmp \quad  c=\cvmp ,
\end{equation*}
for any $\frac{m \ln (N_r/\delta)}{K (e-2)} \leq T$ and $\delta \in [0,1]$, it ensures that
\begin{align}
G_{max}-G_{Exp4.MP} &\leq 2 \sqrt{mKT \ln \frac{N_r}{\delta}} \nonumber \\
&+ 4 \sqrt{mKT \ln \frac{N_r}{m}} + m \ln \frac{N_r}{\delta} \label{regretexp4m2}
\end{align}
holds with probability at least $1-\delta$.
\end{corollary}
\begin{proof}See Appendix B.
\end{proof}
In the next theorem, we show that in the MAB-MP with expert-advice setting, no strategy can enjoy smaller regret guarantee than {\small$O(\sqrt{mKT \ln N_r / \ln K})$} in the minimax sense. In its following, we also show that the derived lower bound is tight and it matches the regret bound of \textit{Exp4.MP} given in Corollary III.1.
\begin{theorem}
\label{lb}
Assume that $N_r=K^n$ for an integer $n$ and that $T$ is a multiple of $n$. Let us define the regret of an arbitrary forecasting strategy ALG in a game length of $T$ as
\begin{equation}
R_{ALG}(T)=G_{max} - G_{ALG}.
\end{equation}
Then there exists a distribution for gain assignments such that 
\begin{gather}
\inf_{\textrm{ALG}} \sup_{\bm{\xi}}R_{ALG}(T)  \geq O\Big(\sqrt{ \frac{mKT \ln N_r }{\ln K} }  \Big),
\end{gather} 
where  $\inf_{\textrm{ALG}}$ is an infimum over all possible forecasting strategies, $\sup_{\bm{\xi}}$ is a supremum over all possible expert advice sequences.
\end{theorem}
\begin{proof}
The presented proof is a modification of \cite[Theorem 1]{lower_bound} for the MAB-MP with expert advice setting. To derive a lower bound for the MAB-MP with expert advice setting, we split the interval $\{1 , \cdots, T \}$ into $n$ non-overlapping subintervals of length $T/n$, where each subinterval is assumed independent and indexed by $k \in \{ 1 , \cdots, n \}$. For each subinterval, we design a MAB-MP game, where the optimal policy is some different $A_k \in \textbf{C}([K],m)$. We also design $N_r = K^n$ sequences of underlying expert advice, such that for every possible every possible sequence of arms $j_1, \cdots, j_n \in \{1, \cdots, K\}^n$, there is an underlying expert that recommends the arms from the sequence throughout the corresponding subintervals. By using the lower bound $O(\sqrt{mKT})$ for the vanilla MAB-MP setting \cite{Audibert2012}, for each subinterval $k$, we have
\begin{gather*}
\inf_{\textrm{ALG}}  R^k_{ALG}(T/n)  \geq O\Big(\sqrt{ \frac{mKT}{n} }  \Big).
\end{gather*} 
where $R^k_{ALG}(T/n)$ is the regret bound corresponding to the
subinterval $k$. By summing all the regret components in each subinterval, and noting $R_{ALG}(T) \geq \sum_{k=1}^m R^k_{ALG}(T/n)$ and $n=\ln N_r/\ln K$, we obtain
\begin{gather*}
\inf_{\textrm{ALG}} \sup_{\bm{\xi}}R_{ALG}(T)  \geq O\Big(\sqrt{ \frac{mKT \ln N_r }{\ln K} }  \Big).
\end{gather*} 
\end{proof}
We note that for $N_r = K$, MAB-MP with expert advice can be reduced to the vanilla MAB-MP setting (by considering underlying experts as arms) and in this case our regret lower bound matches the lower and upper
bounds for the MAB-MP shown in \cite{Audibert2012}. Therefore, we maintain that our lower bound is tight. Furthermore, we note that the regret bound of \textit{Exp4.MP} in Corollary III.1  matches the presented lower bound with an additional $\ln K$ term, while the state-of-art \cite{Kale2010} provides a suboptimal regret bound $O(\sqrt{m KT \ln N})$ (notably for $N_r <<N$ as in the vanilla MAB-MP and the tracking the best $m$-arm settings). Therefore, we state that \textit{Exp4.MP} is an optimal algorithm and it is required to obtain the improvements presented in this paper.

In the following remark, we improve the best-known high-probability bound\cite{Neu2016} for the vanilla $K$-arm multi-play setting by $O(\sqrt{m})$. We note that the resulting bound matches with the minimax lower bound in the soft-Oh sense. Therefore, it cannot be improved in the practical sense.
\begin{remark}
\label{genmabmp}
If we use constant and deterministic actual advice vectors in Exp4.MP, i.e., $\bm{\xi}^k(t)= \textbf{1}_A \in \mathbbm{R}^K$ where $A \in \textbf{C}([K],m)$, the algorithm becomes a vanilla K-armed MAB-MP algorithm. We note that in this scenario, we can directly operate with $\bm{\zeta} ^i(t)=\textbf{1}_i \in \mathbbm{R}^K$, where $N_r=K$. By Corollary \ref{probcor}, if we use {\small$\gamma= \sqrt{\frac{K \ln(K/m)}{mT}}$} and $c=\sqrt{m \ln(K/\delta)}$, Exp4.MP guarantees the regret bound $O(\sqrt{mKT\ln(K/\delta)})$ with probability at least $1- \delta$. Since the most expensive operation of this scenario is arm capping, our algorithm achieves this performance with $O(K \log K)$ time and $O(K)$ space.
\end{remark}

\section{Competing Against the Switching Strategies}
\label{sec:exp3ms}
\begin{algorithm}[t!]
\algsetup{linenosize=\small}
\small
	\caption{Exp3.MSP}\label{alg:algexp3ms}
	\begin{algorithmic}[1]
		\STATE \textbf{Parameters:} $\eta, \gamma, \beta \in [0,1]$ and $c \in R^+$
		\STATE \textbf{Init:} $v_1(j)=1/K$ for $j \in [K]$ 
		\FOR{$t=1$ \TO $T$}
		\IF{$\argmax_{j \in [K]} v_j(t) \geq  \frac{(1/m) - (\gamma/K)}{(1-\gamma)}$} \label{alg2:cap_init}
		\STATE  {Decide $\alpha_t$ as
		$\frac{\alpha_t}{\sum\limits_{v_j(t) \geq \alpha_t} \alpha_t + \sum\limits_{v_j(t) < \alpha_t} v_j(t) } = \frac{(1/m) - (\gamma/K)}{(1-\gamma)}$}
		\STATE Set $U_0(t)= \lbrace	 j: v_j(t) \geq \alpha_t \rbrace$
		\STATE $v'_j(t)=\alpha_t$ for $j \in U_0(t)$
		\ELSE
		\STATE Set $U_0(t) = \emptyset$
		\ENDIF 
		\STATE Set $v'_j(t) = v_j(t)$ for $j \in [K]- U_0(t)$  \label{alg2:cap_fin}
		\STATE $p_j(t)= m\Big((1-\gamma) \frac{v'_j(t)}{\sum\limits_{l=1}^{K} v'_l(t)} + \frac{\gamma}{K}\Big) $ for $j \in [K]$
		\STATE Set $U(t)=$ DepRound$(m,(p_1(t),\cdots,p_K(t)))$ 
		\STATE Observe and receive rewards $x_j(t) \in [0,1]$ for each $j \in U(t)$
	    \STATE $\hat{x}_j(t) = x_j(t)/p_j(t)$  for $j \in U(t)$
		\STATE $\hat{x}_j(t) = 0$   for $j \in [K] - U(t)$
		\FOR{$j=1$ \TO $K$}
		\IF{$j \in [K] - U_0(t)$}
		\STATE $\tilde{v}_j(t)=v_j(t)\exp \Big(\eta (\hat{x}_j(t)+ \frac{c}{p_j(t)\sqrt{KT}}) \Big)$
		\ELSE
		\STATE  $\tilde{v}_j(t)=v_j(t)$
		\ENDIF
		\ENDFOR	
		\STATE $v_j(t+1)= \frac{(1-\beta) \tilde{v}_j(t) + \frac{\beta}{K-1} \sum_{i \not = j} \tilde{v}_i(t)}{\sum_{l=1}^K \tilde{v}_l(t)}$ for $j\in [K]$
		\ENDFOR
	\end{algorithmic}
\end{algorithm}

In this section, we consider competing against the switching $m$-arm strategies. We present \textit{Exp3.MSP}, shown in Algorithm \ref{alg:algexp3ms}, which guarantees to achieve the performance of the best switching $m$-arm strategy with  the minimax optimal regret bound.

We construct \textit{Exp3.MSP} algorithm by using \textit{Exp4.MP} algorithm. For this, we first consider a hypothetical scenario, where we mix each possible $m$-arm selection strategy as an actual expert in \textit{Exp4.MP}. We point out that the  actual advice vectors will be a repeated permutation of the vectors $\{ \bm{1}_A \in R^K: A \in \textbf{C}([K],m)\}$  at each round, which we can write as the sum of $m$ sized subsets of the set $\{ \textbf{1}_i \in R^K: i \in [K]\}$. Therefore, in this hypothetical scenario, we can directly combine all possible  single arm sequences as the underlying experts, where $N_r=K^T$. However, since the regret bound of \textit{Exp4.MP} is $O(\sqrt{mKT\ln(N_r/\delta)})$, a straightforward combination of $K^T$ underlying experts produces a non-vanishing regret bound $O(T)$. To overcome this problem, we will assign a different prior weight  for each one of $K^T$ strategies based on its complexity cost, i.e., the number of segments $S$ (more detail will be given later on).

Let $\textbf{s}_t$ be a sequence of single arm selections, $\textbf{s}_t=\{s_1,s_2,\cdots,s_t\}$ where $\st(t)= s_t \in [K]$, and $\wst$ be its corresponding weight. For ease of notation, we define
\begin{equation}
\nst=\Big(\xhst + \frac{c}{p_{\st(t)}\sqrt{KT}} \Big) \1_{\st(t) \not \in U_0(t)}
\end{equation}
where $p_{\st(t)}$ is the probability of choosing $\st(t)$ at round $t$, and 
\begin{equation}
\Nstm=\sum_{\tau=1}^{t-1} \nstau
\end{equation}
where $\st(i:j)$ denotes the $i^{th}$ through $j^{th}$ elements of the sequence $\st$. Then, the weight of the sequence $\textbf{s}_t$ is given by
\begin{equation}
\label{realprior}
\wst= \pst \exp(\eta \Nstm ),
\end{equation}
where $\pst$ is the prior weight assigned to the sequence $\textbf{s}_t$. 

We point out that using non-uniform prior weights, i.e., $\pst$, is required to have a vanishing regret bound since the number of single arm sequences grows exponentially with $T$. As noted earlier, by Corollary \ref{probcor}, the regret of \textit{Exp4.MP} with uniform initialization is dependent on the logarithm of $N_r$. Therefore, combining $K^T$ strategies with uniform initialization results in a linear regret bound, which is undesirable (the average regret does not diminish).  In order to overcome this problem, similar to complexity penalty of AIC\cite{AIC} and MDL\cite{MDL}, we assign different prior weights $\pi_{\textbf{s}_t}$ for each strategy $\textbf{s}_t$ based on its  the number of segments $S$ . To get a truly online algorithm, we use a
sequentially calculable prior assignment scheme that only depends on the last arm selection such that 
\begin{align}
\label{priorweight}
\psst=  \left\{ \arraycolsep=1.1pt\def\arraystretch{1.3}
\begin{array}{ll}
\frac{1}{K} \quad & \textrm{if} \: t=1 \\
 1-\beta  & \textrm{if} \: \textbf{s}_t(t)= \textbf{s}_t(t-1) \: \: \textrm{(no switch)} \\
\frac{\beta}{K-1} & \textrm{if} \: \textbf{s}_t(t) \neq \textbf{s}_t(t-1) \: \: \textrm{(switch)}.
\end{array} \right.
\end{align}
With the assignment scheme in (\ref{priorweight}), the prior weights are sequentially calculable as $$\pst=\psst \pi_{\st(1:t-1)}$$ and the weights of the arm selection strategies are given by 
\begin{equation}
\label{sweight}
\wst=\psst w_{\textbf{s}_t(1:t-1)} \exp( \eta \nstm ).
\end{equation}\par
In the following theorem, we show that \textit{Exp4.MP} algorithm running with $K^T$ underlying experts with the prior weighting scheme given in (\ref{priorweight}) and (\ref{sweight}) guarantees the  minimax optimal regret bound up to logarithmic factors with probability at least  $1- \betave \delta$. 
\begin{theorem}
\label{switchcor}
If Exp4.MP uses the prior weighting scheme given in (\ref{priorweight}) and (\ref{sweight}),  and the parameters
{\small \begin{alignat*}{2}
&\eta=\etav &&\beta=\betav \\
&\gamma=\gammavms \quad \quad   &&c=\cvms  
\end{alignat*}}to combine all possible single arms sequences as the underlying experts, for any $\frac{m \Sw}{(e-2)K} \lnsd \leq T$ and $\delta \in [0,1]$
{\small \begin{align}
G_{\textbf{M}^*_T}-G_{Exp4.MP} &\leq 6 \sqrt{m \Sw K T \lnsd} \nonumber \\
&+ m \Sw \lnsd
\end{align}}
holds with probability at least $1- \betave \delta$.
\end{theorem}
\begin{proof}
See Appendix C.
\end{proof}
Although we achieved minimax performance, we still suffer from exponential time and space requirements. In the following theorem, we show that by keeping $K$ weights and
updating the weights as
\begin{equation}
\label{shareupdate}
v_j(t+1)= \frac{(1-\beta) \tilde{v}_j(t) + \frac{\beta}{K-1} \sum_{i \not = j} \tilde{v}_i(t)}{\sum_{l=1}^K \tilde{v}_l(t)},
\end{equation}
where
{\small 
\begin{align}
\label{sharevalues}
\tilde{v}_j(t)=  \left\{ \begin{array}{ll}
v_j(t)\exp \Big(\eta (\hat{x}_j(t)+ \frac{c}{p_j(t)\sqrt{KT}})\Big)  & \textrm{if} \: j \in [K]-U_0(t)  \\
v_j(t)  & \textrm{otherwise},
\end{array} \right.
\end{align}
}we can efficiently compute the same weights in the hypothetical \textit{Exp4.MP} run with a computational complexity linear in $K$. To show this, we extend \cite[Theorem 5.1]{CBianchi2006} for the MAB-MP problem:
\begin{theorem}
\label{theoem51ext}
For any $\beta, \gamma, \eta \in [0,1]$, and for any $c, T>0$, \textit{Exp4.MP} algorithm that mixes all possible single arm sequences as underlying experts with the weighting scheme given
in (\ref{priorweight}) and (\ref{sweight}) has equal arm weights with \textit{Exp3.MSP} algorithm, which
updates its weights according to formulas in (\ref{shareupdate}) and (\ref{sharevalues}).
\end{theorem}
\begin{proof}
See Appendix C.
\end{proof}
Theorem \ref{theoem51ext} proves that for any parameter selection \textit{Exp3.MSP} is equivalent algorithm to the hypothetical \textit{Exp4.MP} run. Therefore, Theorem \ref{switchcor} is valid for \textit{Exp3.MSP}, which shows that \textit{Exp3.MSP} has a $\tilde{O}(\sqrt{mSKT})$ regret bound holding with at probability at least $1-\frac{S-1}{e(T-1)}$ with respect to the optimal $m$-arm strategy. Since the most expensive operation in the algorithm is capping, \textit{Exp3.MSP} requires $O(K \log K)$ time complexity per round.

\section{Experiments}\label{sec:experiments}
\begin{figure*}[t!]
    \centering
    \begin{tabular}{cc}
    \begin{subfigure}[b]{0.5\textwidth}
        \centering
        \includegraphics[width=0.65\textwidth]{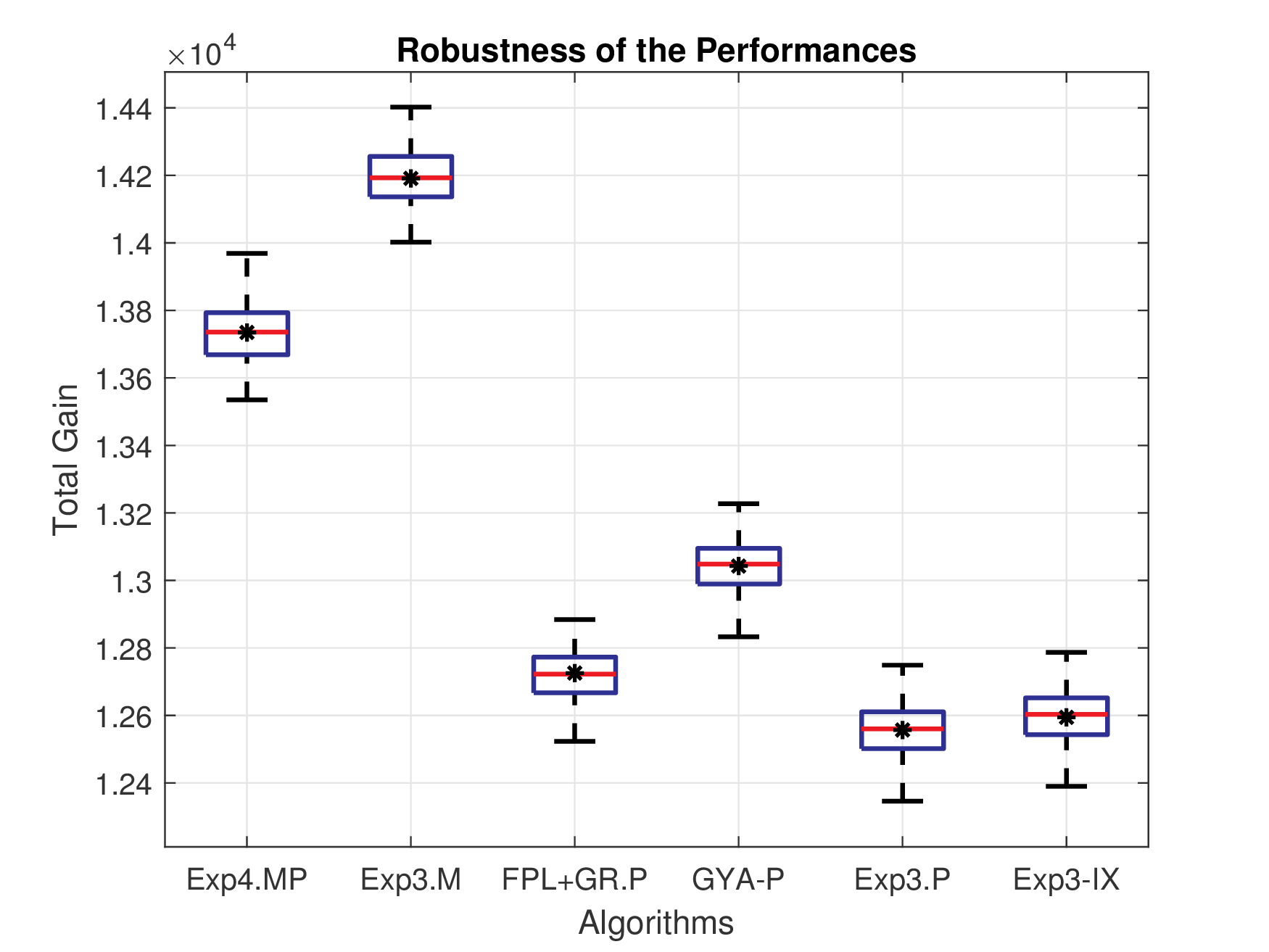}\\
        \caption{}\label{fig:up_to_T2}
    \end{subfigure}
    \begin{subfigure}[b]{0.5\textwidth}
        \centering
        \includegraphics[width=0.65\textwidth]{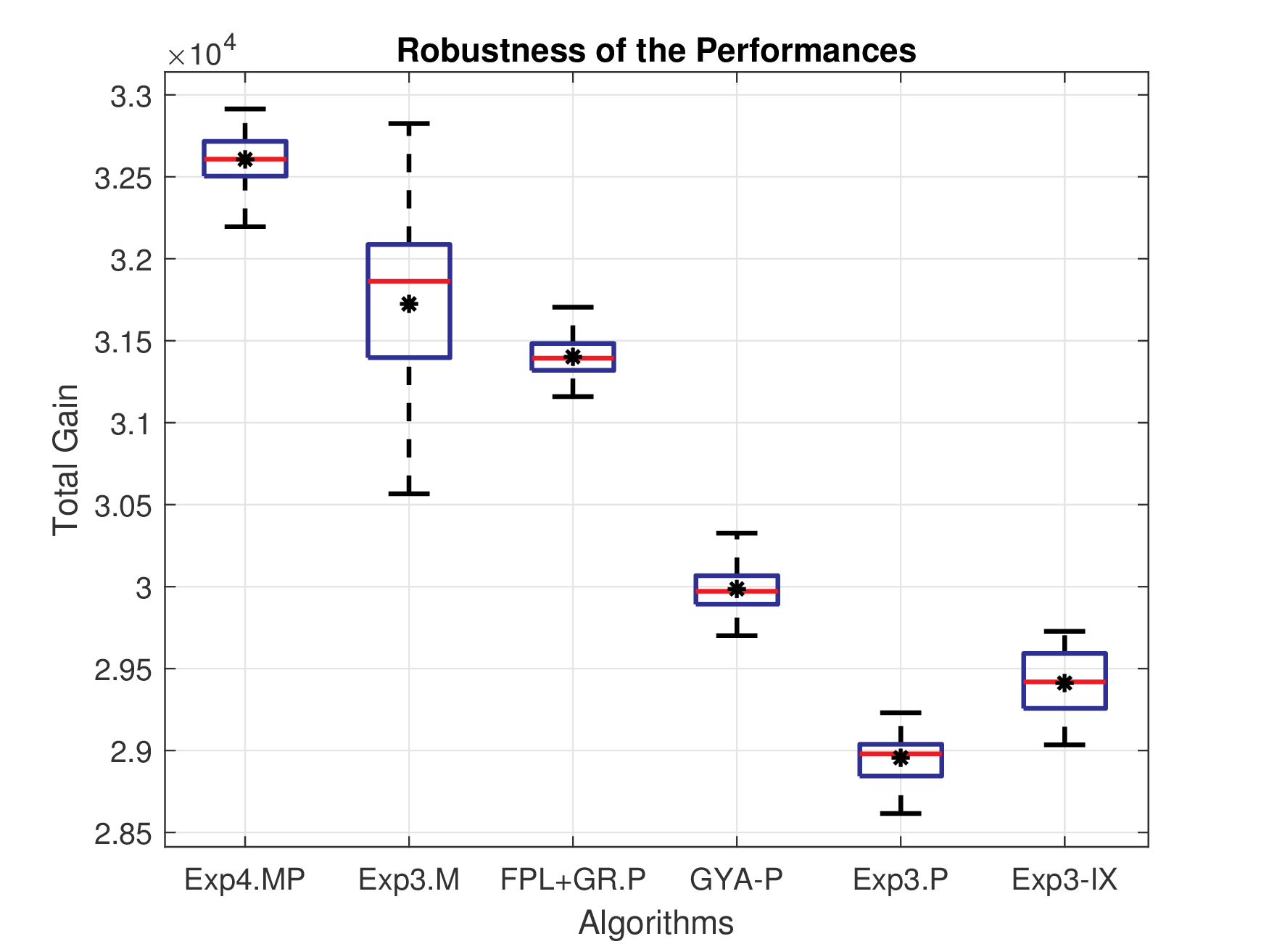}\\
        \caption{}\label{fig:up_to_T}
    \end{subfigure}
    \end{tabular}
    \caption{Comparison of the distributions of the total gains received by the algorithms (a) up to $T/2$ (b) up to $T$.}\label{fig:variance}
\end{figure*}
In this section, we demonstrate the performance of our algorithms with simulations on real and synthetic data.  These simulations are mainly
meant to provide a visualization of how our algorithms perform in comparison to the state-of-the-art techniques and should not be seen as verification of the mathematical results in the previous sections.  We note that simulations
only show the loss/gain of an algorithm for a typical sequence of examples; however, the mathematical results of
our paper are the worst-case bounds that hold even for adversarially-generated sequences of examples. \par
For the following simulations, we use four synthesized datasets and one real dataset. We compare \textit{Exp4.MP} with \textit{Exp3.P}~\cite{Auer1995}, \textit{Exp3-IX} \cite{Neu15},  \textit{Exp3.M}~\cite{Uchiya2010}, \textit{FPL
+GR.P}~\cite{Neu2016}, \cite[Figure 2]{Gyorgy2007}, and \cite[Figure 4]{Kale2010}. We compare \textit{Exp3.MSP} with \cite[Figure 4]{Gyorgy2007}, and \textit{Exp3.S} \cite{Auer1995}. We note that all the simulated algorithms are
constructed as instructed in their original publications. The parameters of the individual algorithms are set as instructed by their respective publications. The information of the game length $T$ and the number of the segments in the best strategy $\Sw$ have been 
given a priori to all algorithms. In each subsection, all the compared algorithms are presented to the identical games.
	
\subsection{Robustness of the Performances}
We conduct an experiment to demonstrate the robustness of our high-probability algorithms. For this, we 
run algorithms several times and compare the distributions of their total gains. For the comparison, we use 
\textit{Exp4.MP} with the deterministic and constant advice vectors. We 
compare \textit{Exp4.MP} with \textit{Exp3.P}~\cite{Auer1995}, \textit{Exp3-IX} \cite{Neu15},  \textit{Exp3.M}~\cite{Uchiya2010}, \textit{FPL
+GR.P}~\cite{Neu2016} and the high-probability algorithm introduced by Gy{\"o}rgy et al.\ in~\cite{Gyorgy2007}. 
Since Gy{\"o}rgy et al.\ did not name their algorithms, we use \textit{GYA-P} to denote their high-probability 
algorithm. We highlight that all the algorithms except \textit{Exp3.M} guarantee a regret bound with high 
probability, whereas \textit{Exp3.M} guarantees an expected regret bound. \par
For this experiment, we construct a $10$-arm bandit game where we choose $5$ bandit arms at each round. All gains are generated by independent draws of Bernoulli random variables. In the first half of the
game, the mean gains of the first $5$ arms are $0.5 + \epsilon$, and the mean gains of others' are $0.5 - \epsilon$. In the second half of the game, the mean gains of the first $5$ arms have been reduced to $0.5 - \epsilon$, while the mean gains of others' have been increased to $0.5 + 4\epsilon$. We point out that based on these selections, the 
$m$-arm consisting of the last $5$ arms performs better than the others in the full game length. \par
We set the parameters $T=10^4$, $\epsilon=0.1$ for all the algorithms, and $\delta=0.01$ for the high-probability algorithms. Our experiments are repeated $100$ times to obtain statistically significant results. Since the game environments are the same for all the algorithms, we directly compare the total gains. We study the total gain up to two interesting rounds in the game: up to $T/2$, where the losses are independent and identically distributed, and up to $T$, where the algorithms have to notice the shift in the gain distribution. \par
We have constructed box plots by using the resulting total gains of the algorithms. In the box plots, the lines extending from the boxes (the whiskers) illustrate the minimum and the maximum of the data. The boxes extend from the first quartile to the third quartile. The horizontal lines and the stars stand for the median and the mean of the distributions. Fig. \ref{fig:up_to_T2} illustrates the distributions of the total gains up to $T/2$. We observe that the variance of all the total gains are comparable. The mean total gain received by \textit{Exp4.MP} is only comparable with that of \textit{Exp3.M} while outperforms the rest. On the other hand, when the change occurred in the game, 
\textit{Exp4.MP} outperforms the rest in the overall performance (Fig. \ref{fig:up_to_T}). As expected, \textit{Exp3.M} and \textit{Exp4.MP} receive relatively higher gains than the other algorithms. However, since we give a special care for bounding the variance, \textit{Exp4.MP} has a more robust performance. From the results, we can conclude that \textit{Exp4.MP} yields the superior performance of the algorithms with an expected regret guarantee and the robustness of the high-probability algorithms at the same time.
\subsection{Choosing an $m$-arm with Expert Advice}
\begin{figure*}[t!]
    \centering
    \begin{tabular}{ccc}
    \begin{subfigure}[b]{0.3\textwidth}
        \centering
        \includegraphics[width=\textwidth]{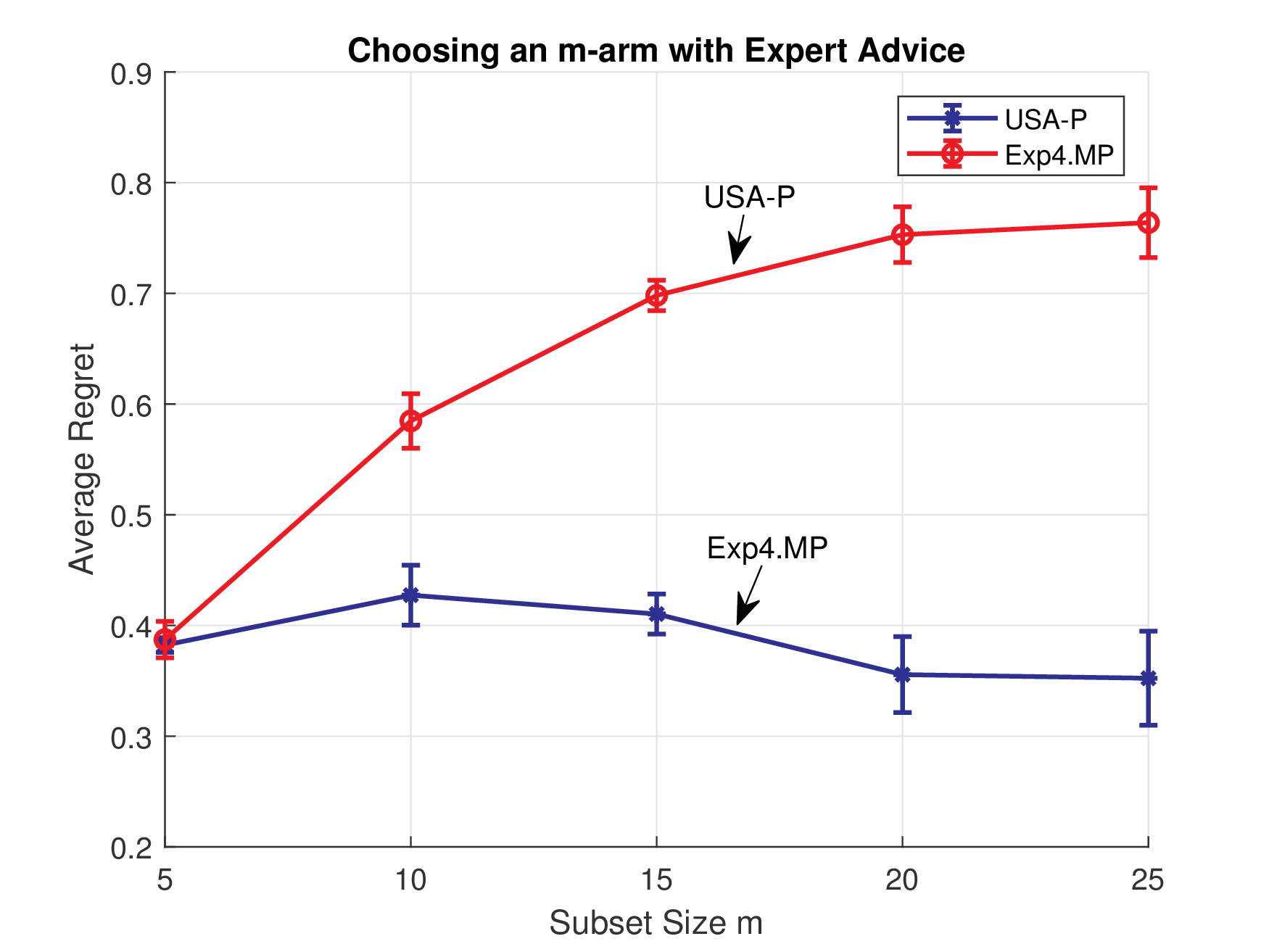}\\
        \caption{}\label{fig:varm}
    \end{subfigure} &
    \begin{subfigure}[b]{0.3\textwidth}
        \centering
        \includegraphics[width=\textwidth]{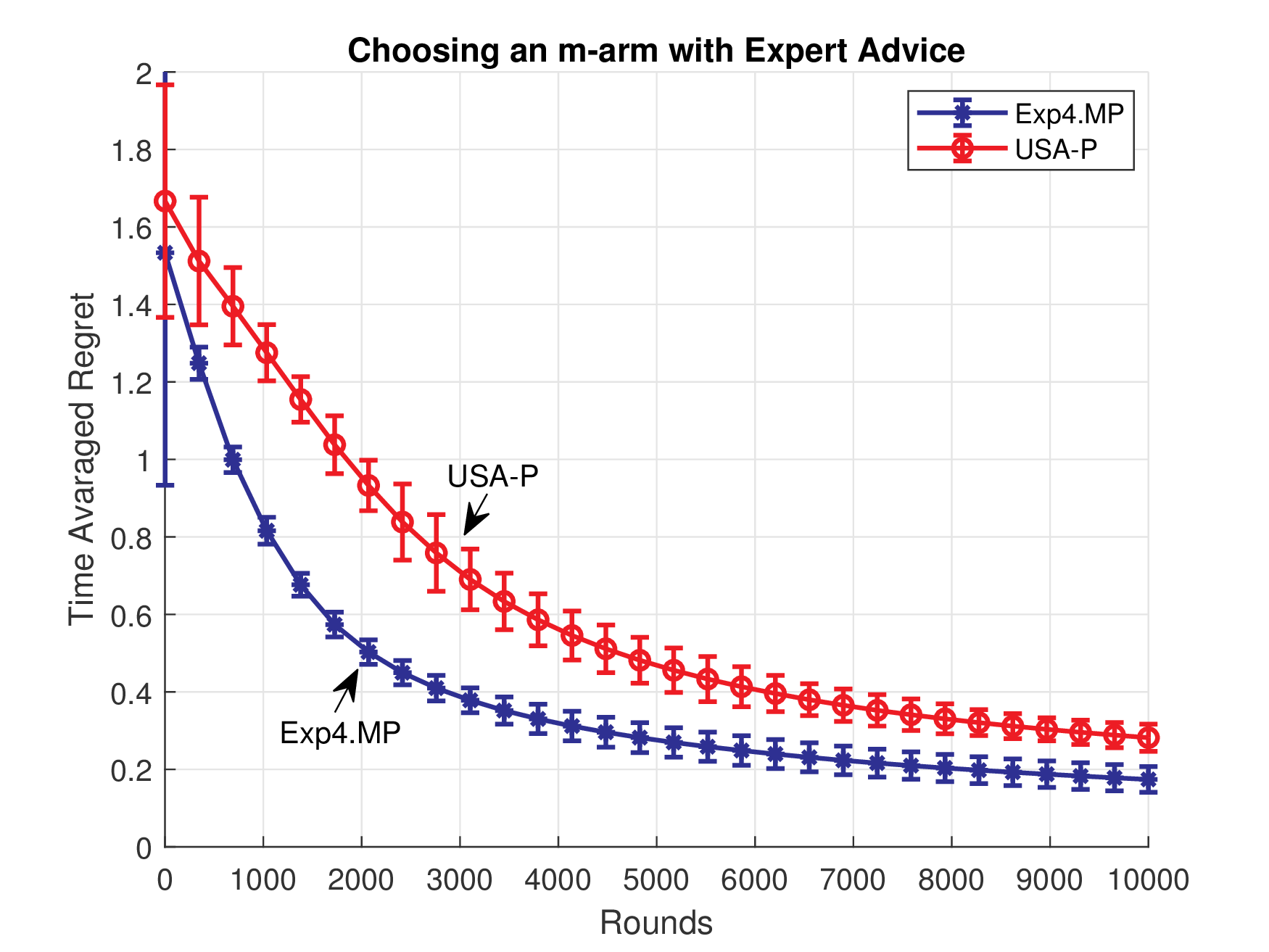}\\
        \caption{}\label{fig:vart}
    \end{subfigure} &
     \begin{subfigure}[b]{0.3\textwidth}
        \centering
        \includegraphics[width=\textwidth]{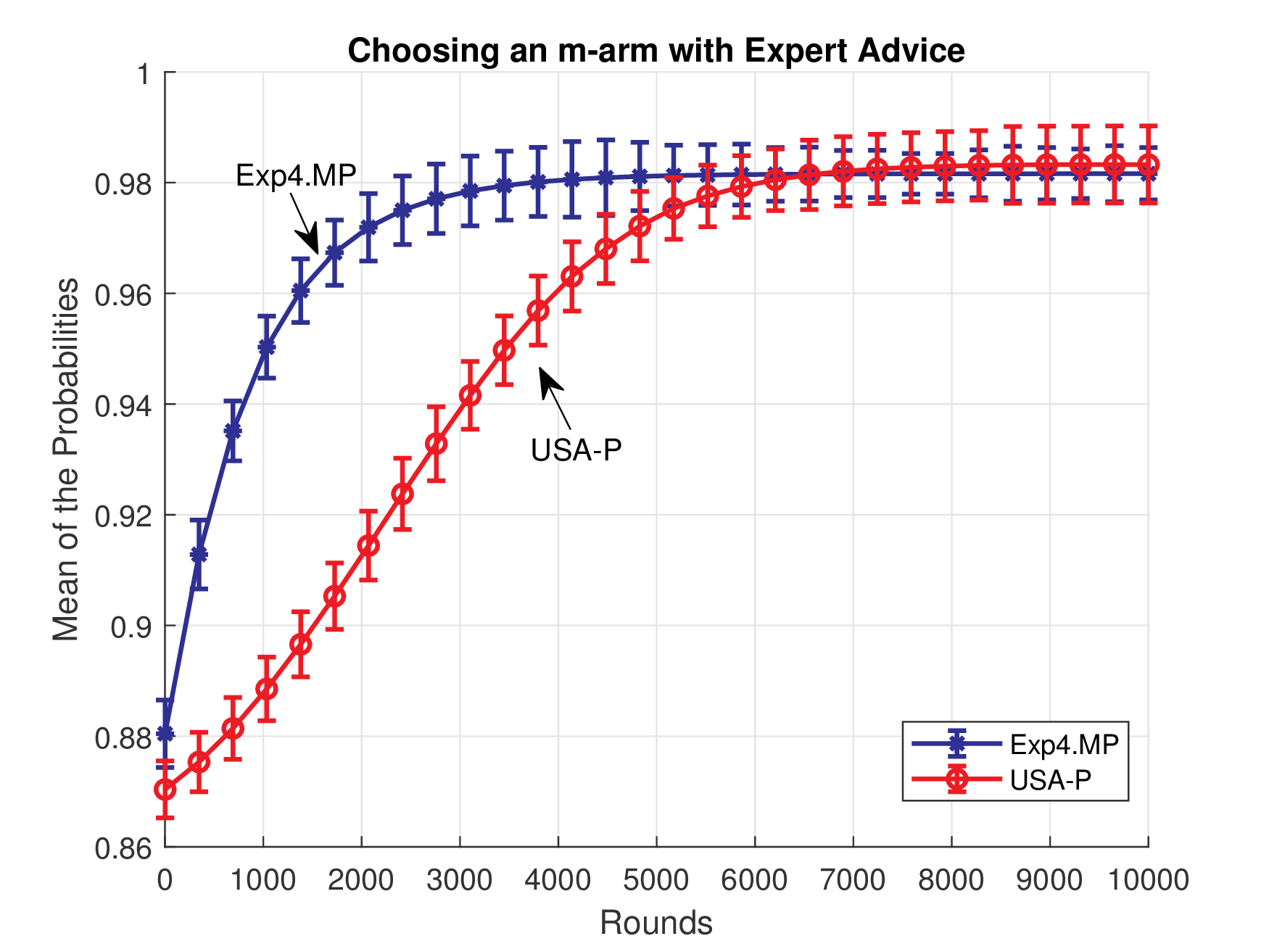}\\
        \caption{}\label{fig:varp}
    \end{subfigure}
    \end{tabular}
    \caption{(a) Per round regret performances of the algorithms with increasing $m$. (b) Time-averaged regret performances of the algorithms in a game where $K=30$, $m=15$. (c) Mean of the probabilities assigned to the optimum $m$ arms by both algorithms when $m=15$.}\label{fig:experts}
\end{figure*}
\begin{figure*}[t!]
\centering
\begin{tabular}{@{}cc@{}}
    \begin{subfigure}[b]{0.5\textwidth}
        \centering
        \includegraphics[width=0.7\textwidth]{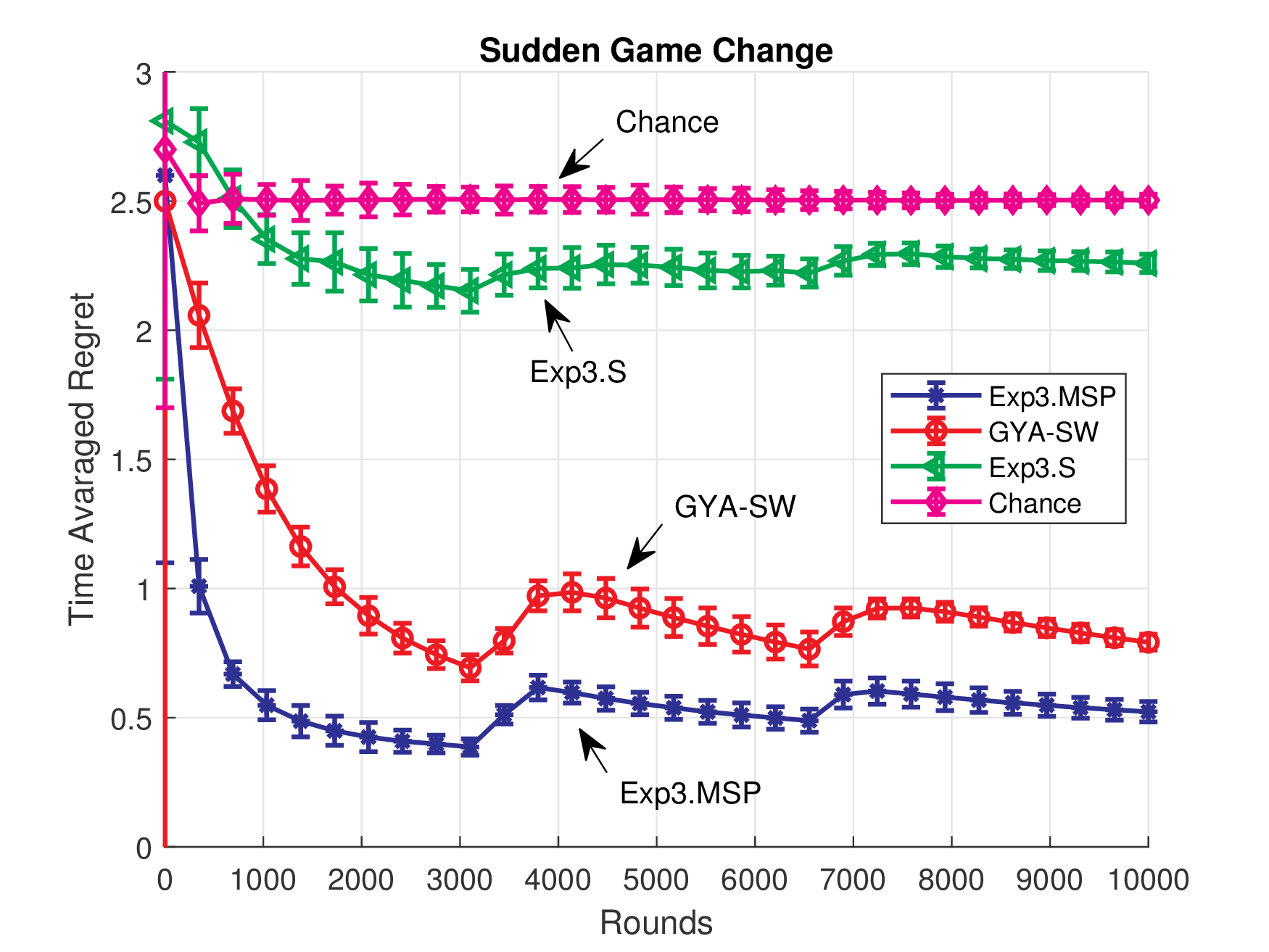}\\
        \caption{}\label{fig:sw_regret}
    \end{subfigure} &
    \begin{subfigure}[b]{0.5\textwidth}
        \centering
        \includegraphics[width=0.7\textwidth]{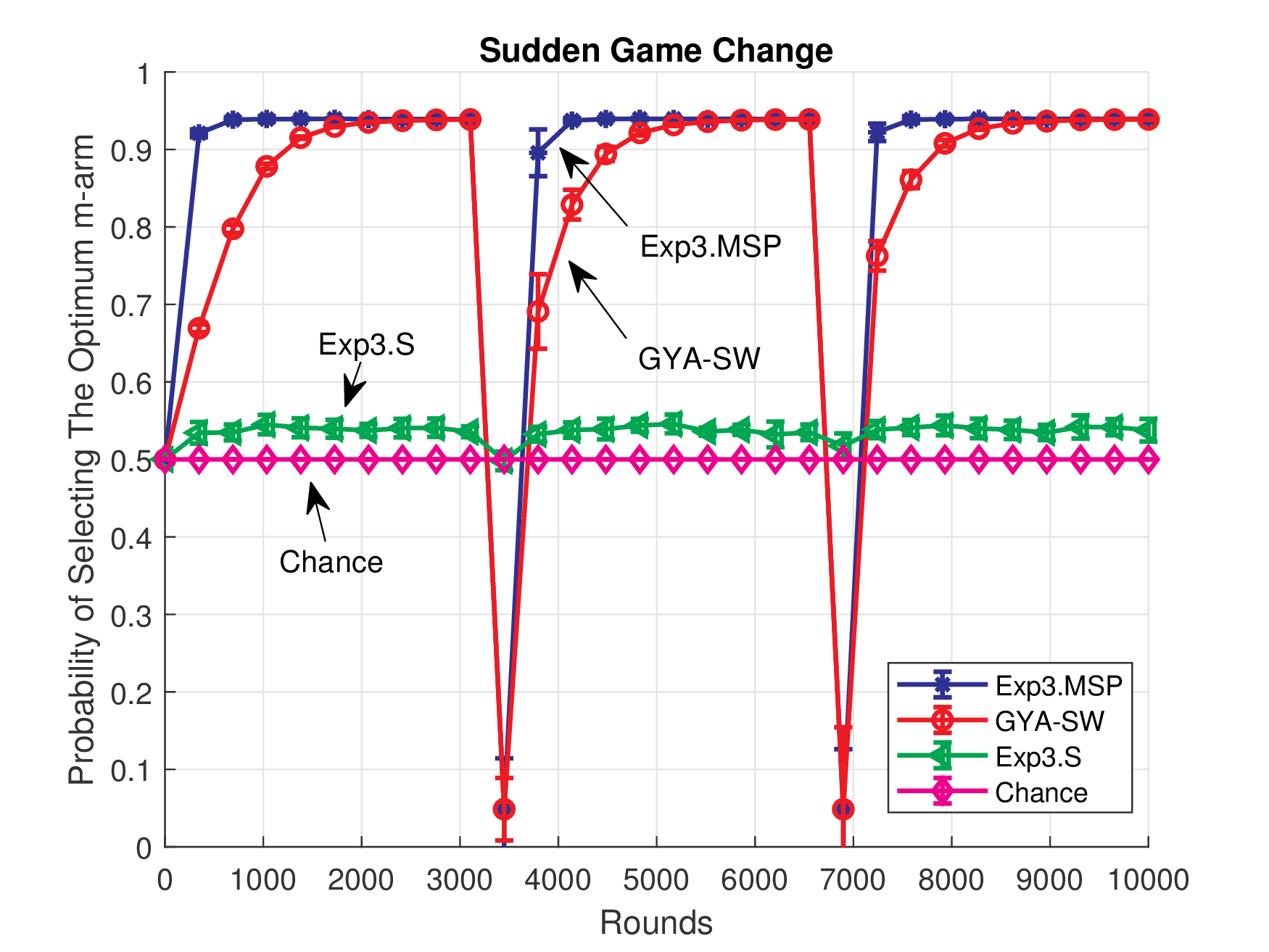}\\
        \caption{}\label{fig:sw_prob}
    \end{subfigure}
    \end{tabular}
    \caption{(a) Time averaged regrets of the algorithms in a game where $K=10$, $m=5$ and the optimum $m$-arm changes at every $3333$ rounds. (b) Probabilities of
selecting the optimum $m$-arm for all of the algorithms in sudden game change setting.}\label{fig:experts}
\end{figure*}
In this part, we demonstrate the performance of \textit{Exp4.MP} when the advice vectors of the actual advice vectors are not necessarily constant nor deterministic. We compare our algorithm with the only known algorithm that is capable of choosing $m$-arm with expert advice, i.e., \textit{Unordered Slate Algorithm with policies (USA-P)} introduced in~\cite{Kale2010}.
Since our main point is to improve the regret bound by $O(\sqrt{m})$, we compare algorithms under different subset sizes. For this, we construct five $30$-armed games, where we choose $m \in \{5,10,15,20,25\}$
arms respectively. In each game, the gains of the first $m$ arms are $1$, and the gains of the others are $0$ throughout the game. In order to satisfy the condition $N_r=O(N^{\frac{1}{m}})$, we first generate the underlying expert set where $N_r=m+2$. The generation process is as follows: The first $m$ underlying experts are chosen constant and deterministic where $\bm{\zeta^i}=\bm{1}_i$ for $i\in \{1,\cdots,m\}$. The first $m$ entries of the last two
 underlying experts are chosen $0$, i.e., $\zeta^{m+1}_{j}(t)=\zeta^{m+2}_{j}(t)=0$ for $j\in \{1,\cdots,m\}$, while
the other entries are determined randomly at each round under the constraint that their sum is $1$. The actual vectors are generated 
by summing each $m$ sized subset of the underlying expert set at each round, where we have a total of $\binom{m+2}{m}$ experts. We note that based on our arm gains selection and the advice vector generating process, the actual expert which is the sum of the constant underlying experts is the best expert.\par
In the experiment, we set the parameters $T=10^4$ for both algorithms and $\delta=0.01$ for \textit{Exp4.MP}. We have repeated all the games $100$ times and plotted the ensemble distributions in Fig. \ref{fig:varm}, Fig. \ref{fig:vart}, and Fig. \ref{fig:varp}.
Fig. \ref{fig:varm} illustrates the time averaged regret incurred by the algorithms at the end of the games with increasing $m$. As can be seen, the regret incurred by our algorithm remains almost constant while the regret of \textit{USA-P} increases as $m$ increases. In order to observe the temporal performances, we have plotted Fig.  \ref{fig:vart}, which illustrates the time averaged regret performances of the algorithms
 when $K=30$, and $m=15$. We observe that our algorithm suffers a lower regret value at each round. To analyze this difference in the performances, we have also plotted the mean of the probability values assigned to the optimum $m$ arms by the algorithms at each round when $m=15$ (Fig. \ref{fig:varp}). We observe that \textit{USA-P} saturates at the same probability value as \textit{Exp4.MP}, its convergence rate is slower. Therefore, since our algorithm can explore the optimum $m$-arm more rapidly, it is able to achieve better performance, especially in high $m$ values.
\subsection{Sudden Game Change}
\begin{figure*}[t!]
   \centering
\begin{tabular}{cc}
\begin{subfigure}[b]{0.5\textwidth}
   \centering
\includegraphics[width=0.7\textwidth]{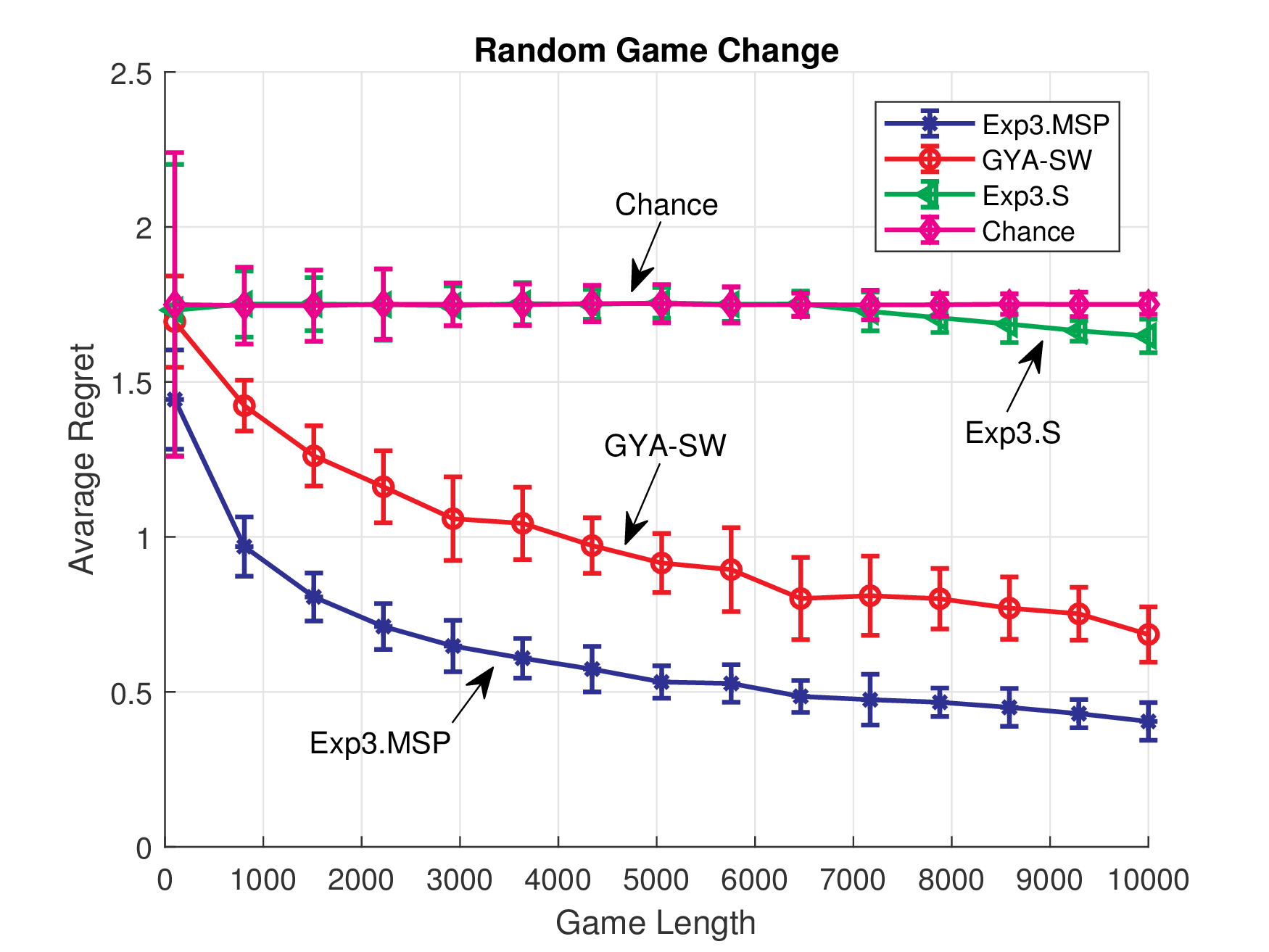}
\caption{} \label{fig:hor}
\end{subfigure}
\begin{subfigure}[b]{0.5\textwidth}
   \centering
\includegraphics[width=0.7\textwidth]{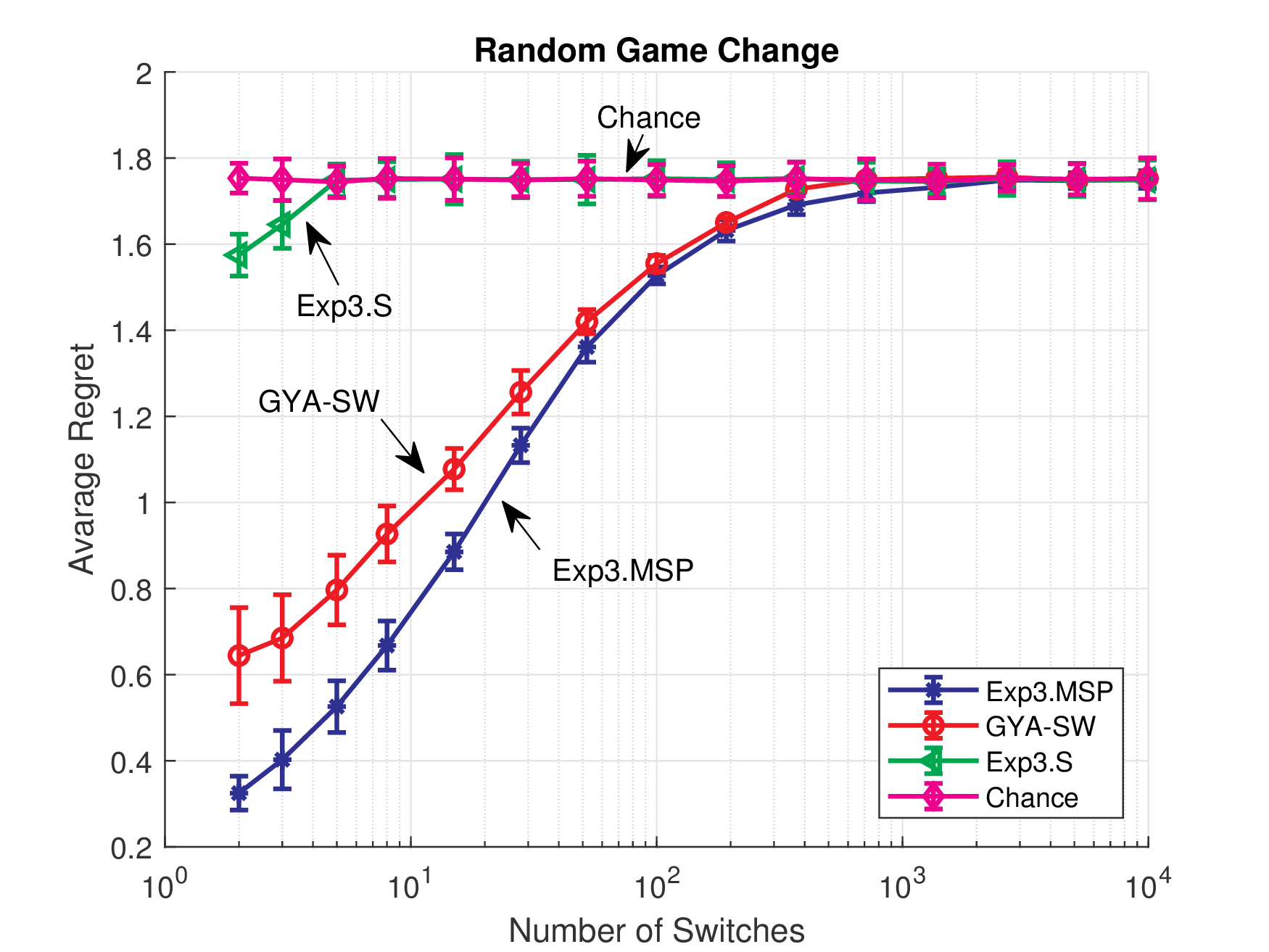}
\caption{} \label{fig:sw}
\end{subfigure}\\
\begin{subfigure}[b]{0.5\textwidth}
   \centering
\includegraphics[width=0.7\textwidth]{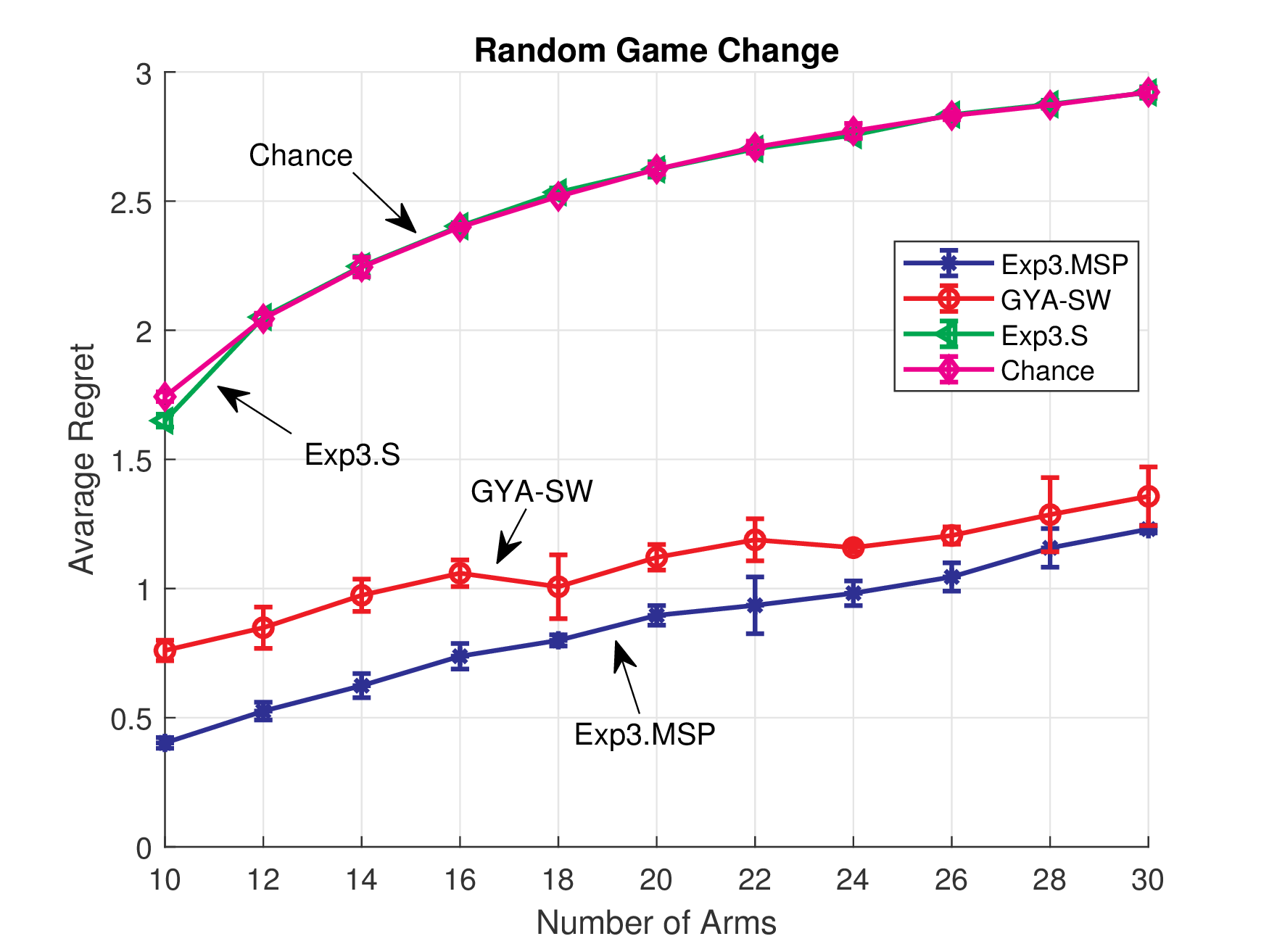}
\caption{} \label{fig:arm}
\end{subfigure}
\begin{subfigure}[b]{0.5\textwidth}
   \centering
\includegraphics[width=0.7\textwidth]{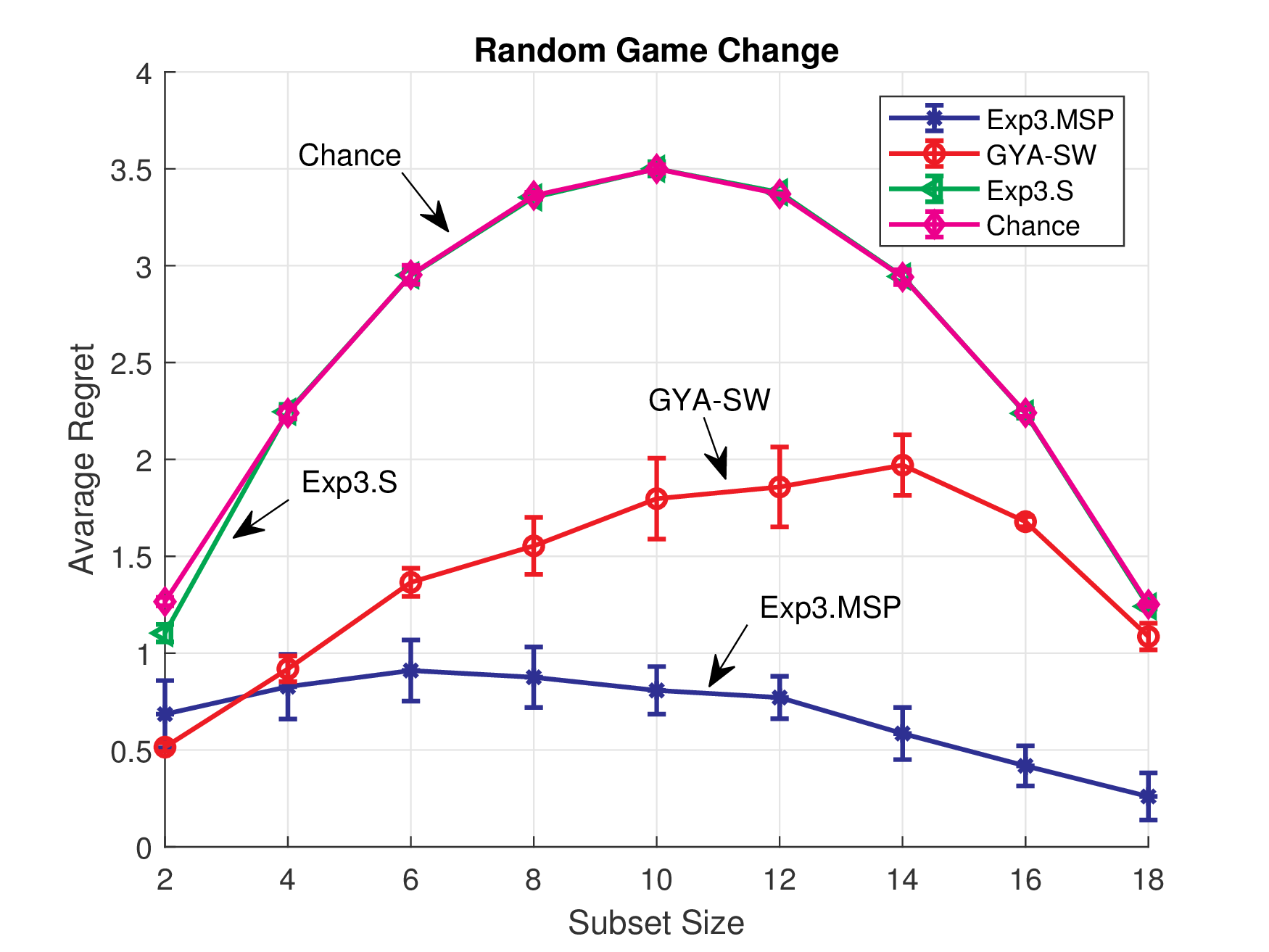}
\caption{} \label{fig:subset}
\end{subfigure}\\
   \end{tabular}
    \caption{(a) Per round regrets of the algorithms with increasing game length. (b) Per round regrets of the algorithms with increasing
number of switches. (c) Per round regrets of the algorithms with increasing number of bandit arms. (d) Per round regrets of the algorithms with increasing number of subset size.}
\end{figure*}
In this section, we demonstrate the performance of \textit{Exp3.MSP} in a synthesized game. We compare our algorithm with Exp3.S and the algorithm introduced by Gy{\"o}rgy et al. in \cite[Section 6]{Gyorgy2007}. Since Gy{\"o}rgy et al.\ did not name their algorithms, we use
\textit{GYA-SW} to denote their algorithm. We also compare each algorithm against the trivial algorithm, Chance (i.e., random guess) for a baseline comparison.  \par
For this experiment, we construct a game of length $T=10^4$, where we need to choose $5$ arms out of $10$ bandit arms. The gains of the arms are deterministically selected as follows: Up to round $3333$, the gains of
the first $5$ arms are $1$, while the gains of the rest are $0$. Between rounds $3334$ and $6666$, the gains 
of the last $5$ arms are $1$, while the gains of the rest $0$. In the rest of the game, the gain distribution is 
the same as in the first $3333$ rounds. The optimum $m$-arms at consecutive segments are intentionally selected mutually exclusive in order to simulate sudden changes effectively. We point out that based on our arm gains selections, the number of segments in the optimum $m$-arm sequence is $3$, i.e., $S=3$. \par
For \textit{Exp3.MSP} and \textit{GYA-SW}, we set $\delta=0.01$. We have repeated the games $100$ times and plotted the ensemble distributions in Fig. \ref{fig:sw_regret} and Fig. \ref{fig:sw_prob}. Fig. \ref{fig:sw_regret} illustrates the time-averaged regret performance of the algorithms. We observe that our algorithm has a lower regret value at any time instance. To analyse this, we have
plotted the mean probability values assigned to the optimum $m$ arms by the algorithms at each round in Fig. \ref{fig:sw_prob}. We observe that \textit{Exp3.S} cannot reach high values of probability due to the exponential size of its action set. We also see that although \textit{GYA-SW} saturates at the same probability value as \textit{Exp3.MSP}, its convergence rate is slower. Therefore, Fig. \ref{fig:sw_prob} shows that since our algorithm adapts  faster to the changes in the environment, it achieves a better performance throughout the game.
\subsection{Random Game Change}
\begin{figure*}[t!]
    \centering
    \begin{tabular}{ccc}
    \begin{subfigure}[b]{0.32\textwidth}
        \centering
        \includegraphics[width=\textwidth]{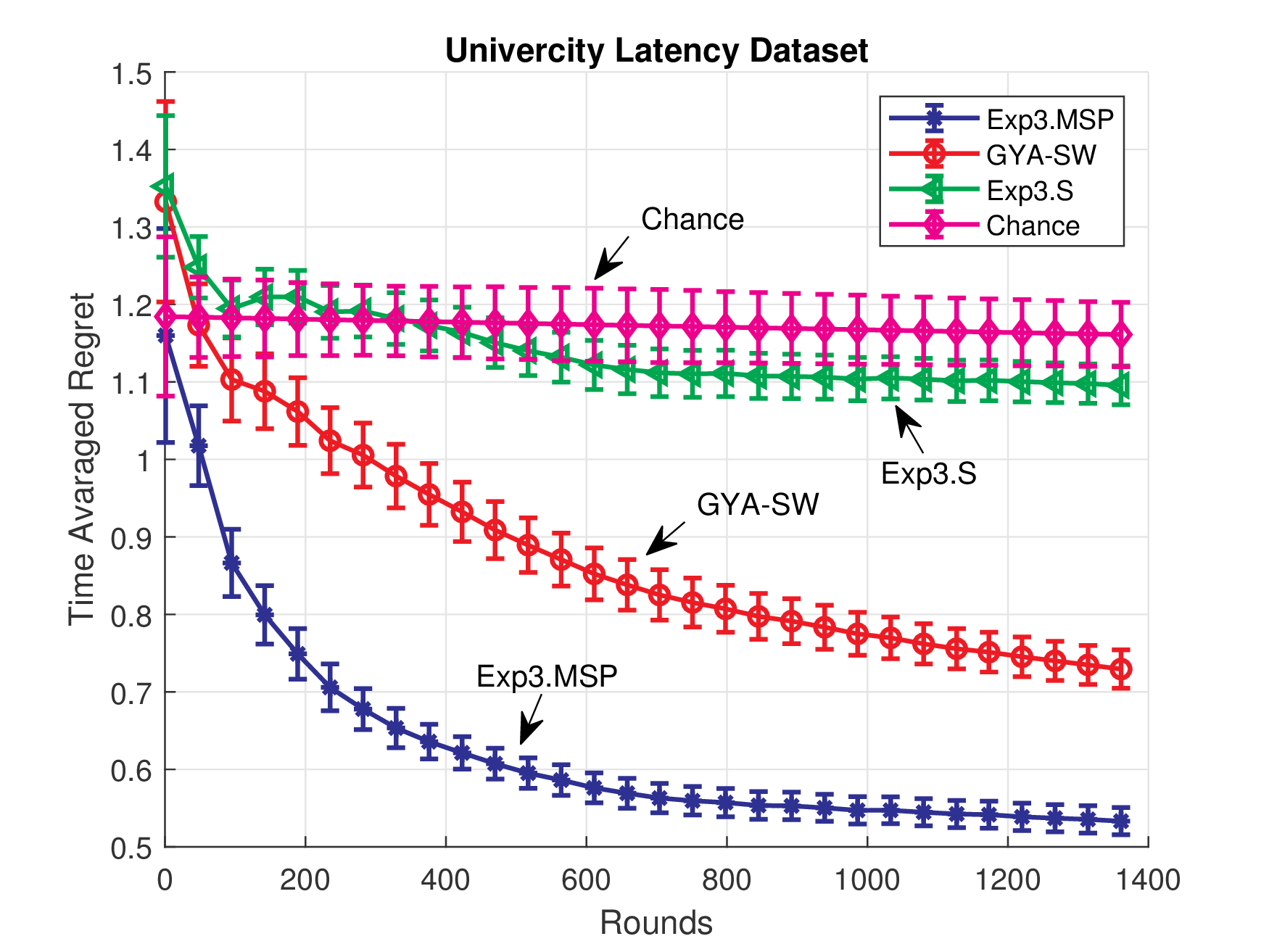}\\
        \caption{}\label{fig:unik}
    \end{subfigure}
    \begin{subfigure}[b]{0.32\textwidth}
        \centering
        \includegraphics[width=\textwidth]{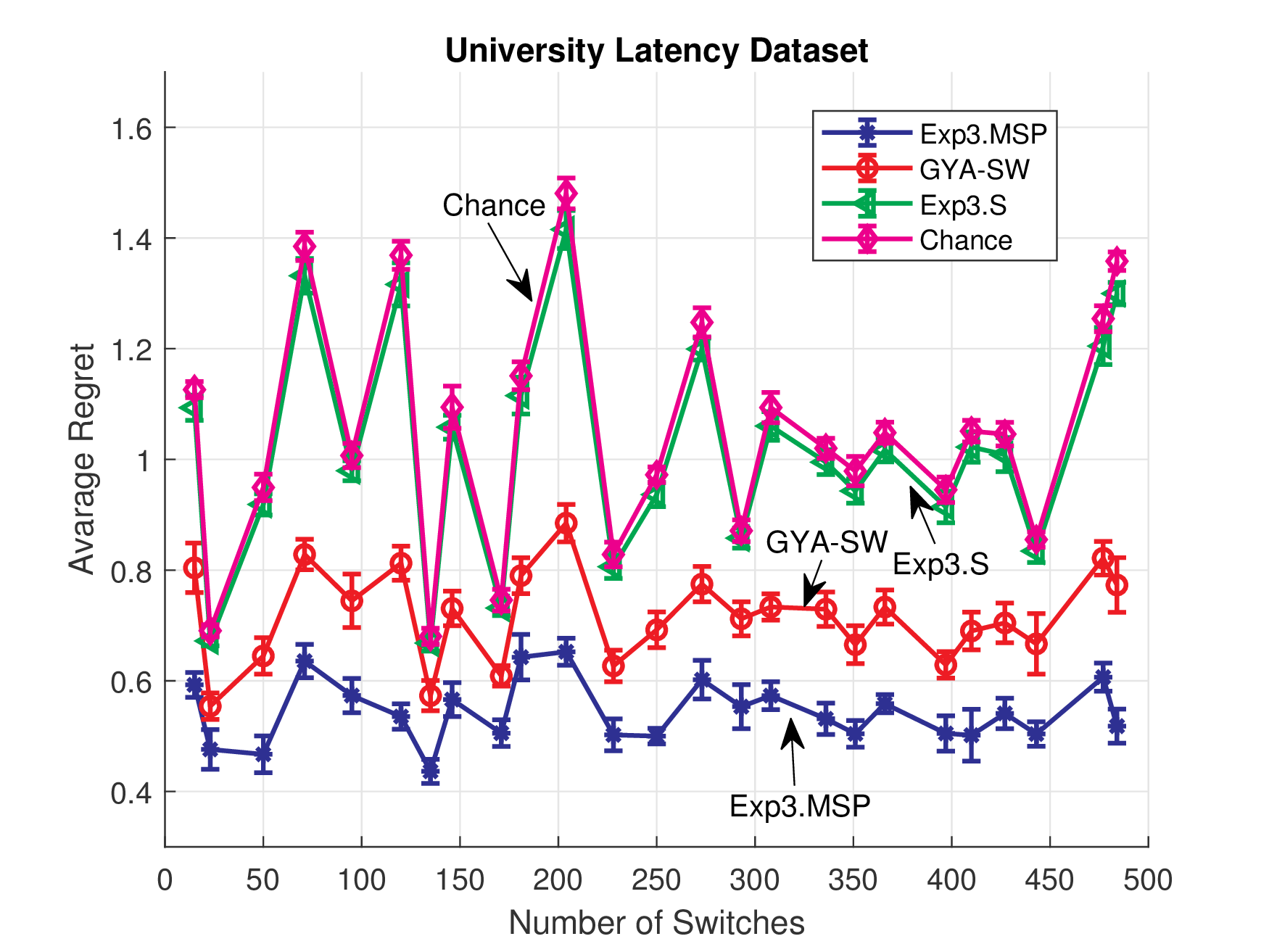}\\
        \caption{}\label{fig:unis}
    \end{subfigure}
     \begin{subfigure}[b]{0.32\textwidth}
        \centering
        \includegraphics[width=\textwidth]{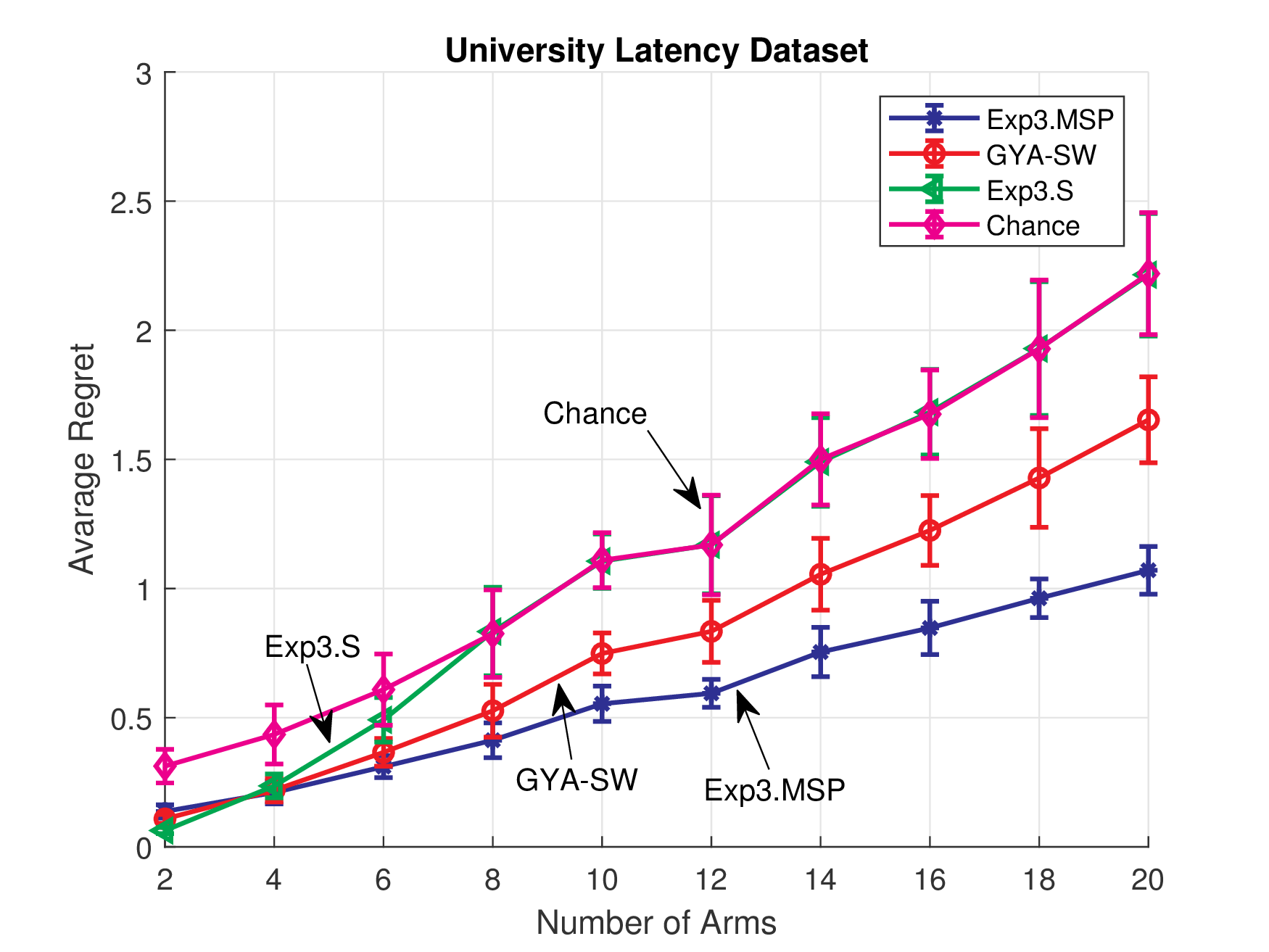}\\
        \caption{}\label{fig:unia}
    \end{subfigure}
    \end{tabular}
    \caption{(a) Time averaged regrets in "univ-latencies" dataset when $K=10$, $m=5$ and $S=3$. (b) Per round regret performances of the algorithms at the end of the games with increasing number of switches. (c) Per round regret performances of the algorithms at the end of the games with increasing number of bandit arms.}\label{fig:uni_lat}
\end{figure*}
In this part, we demonstrate the performance of \textit{Exp3.MSP} on random data sequences. We compare our algorithm with two state-of-the-art techniques: \textit{Exp3.S}\cite{Auer1995} and \textit{GYA-SW}\cite{Gyorgy2007}. We also compare each algorithm against the trivial algorithm, Chance (i.e., random guess) for a baseline comparison.  
For this experiment, we construct a game whose behavior is completely random with the only regularization condition being an $m$-arm should be optimum throughout a segment. We start to synthesize the dataset by randomly selecting gains in $[0,1]$ for all arms for all rounds. We predetermine the optimum $m$-arms in each segment and then switch the maximum gains with the gains of the optimum $m$-arm at each round. This synthesized dataset creates a game with randomly determined gains while maintaining that one $m$-arm is uniformly optimum throughout each segment.  We synthesize multiple datasets to analyze the effects of the parameters of the game individually, where we compare the algorithms’ performances for varying game length $(T)$, number of switches $(S)$, number of arms $(K)$ and subset size $(m)$. We start with the control group of $T = 10^4$, $K = 10$, $m=5$, $S = 3$ and both $T$ and $S$ is known a priori. Then, for each case, we vary one of the above four parameters. Differing from before, the time instances of switches are not fixed to $3333$ and $6666$ but instead selected randomly to be in anywhere in the game. Thus, we create completely random games with three segments. \par
To observe the effect of game length, we selected the $15$ different game lengths, which are linearly spaced between $10^2$ and $10^4$. We provided the algorithms with the
prior information of both the game length and the number of
switches. In Fig. \ref{fig:hor}, we have plotted the average regret incurred at the end of the game, i.e., $R(T)/T$, by the algorithms at different values of game length while fixing the other parameters.
We note that the error bars in Fig. \ref{fig:hor} illustrate the maximum and the minimum average regret incurred at a fixed value of game length. For any set of parameters, we have simulated the setting for $25$ times with recreating the game each time to obtain statistically significant results.  We observe that the algorithm \textit{Exp3.S} performs close to random guess up to approximately game length of $7000$ rounds, which is expected since it assumes each action as a separate arm. We also observe that \textit{Exp3.MSP} and \textit{GYA-SW} perform better than the chance for all values of game length. However, there is a significant performance difference in favor of \textit{Exp3.MSP}. \par
To observe the effect of the number of switches on the
performances, we created random change games with $10$ arms, subset size $m=5$
and game length of $10^4$. We provided the algorithms with the
prior information of both the game length and the number of
switches. We selected $15$ different switch values, which are logarithmically spaced between $2$ and $10^4$. In Fig. \ref{fig:sw}, we have plotted
the average regret incurred at the end of the game by
the algorithms at different values of number of switches while
fixing the other parameters. For any set of parameters, we
have simulated the setting for $25$ times with recreating the
game each time. The algorithm \textit{Exp3.S} performs similar to
random guess after approximately $5$ switches, i.e., $S = 5$.
Both \textit{Exp3.MSP} and \textit{GYA-SW} catch random guess
at $S = 1000$, which is comparable to the value of game length, i.e., $T=10^4$.
However, \textit{Exp3.MSP} manages to provide better performance than the
other algorithms for all number of switches. \par
To observe the effect of the number of bandit arms on the
performances, we created random change games with $3$ segments, subset size of $5$ and game length of $10^4$. We provided the algorithms
with the prior information of both the game length and the
number of switches. We selected the number of bandit arms to
be even numbers between $10$ and $30$. In Fig. \ref{fig:arm}, we have
plotted the average regret incurred at the end of the game by the algorithms at different values of the number of bandit arms
while fixing the other parameters. For any set of parameters,
we have simulated the setting for $25$ times with recreating the
game each time. The algorithm \textit{Exp3.S} performs
similar to random guess after approximately $12$ bandit arms.
\textit{Exp3.MSP} and \textit{GYA-SW} outperform random guess for all values of bandit arms. On the other hand, \textit{Exp3.MSP} outperforms
all algorithms for all values of bandit arms uniformly. \par
To observe the effect of the subset size on the
performances, we created random change games with $3$ segments, $20$ bandit arms and game length of $10^4$. We provided the algorithms
with the prior information of both the game length and the
number of switches. We selected the subset size to
be even numbers between $2$ and $18$. In Fig. \ref{fig:subset}, we have
plotted the average regret incurred at the end of the game by the algorithms at different values of number of subset sizes
while fixing the other parameters. For any set of parameters,
we have simulated the setting for $25$ times with recreating the
game each time. We observe that the algorithm \textit{Exp3.S} performs
similar to random guess after subset size is equal to $4$.
\textit{Exp3.MSP} and \textit{GYA-SW} outperform random guess for all values of subset sizes. On the other hand, \textit{Exp3.MSP} yields better performance,
especially in high values of subset size.
\subsection{Online Shortest Path}
In this subsection, we use a real-world networking dataset that corresponds to the retrieval latencies of the homepages of $760$ universities. The pages were probed every $10$ min for about $10$ days in May $2004$ from an internet connection located in New York, NY, USA \cite{data}. The resulting data includes $760$ URLs and $1361$ latencies (in millisecond) per URL. For the setting, we consider an agent that must retrieve data through a network with several redundant sources available. For each retrieval, the agent is assumed to select $m$ sources and wait until the data is retrieved. The objective of the agent is to minimize the sum of the delays for the successive retrievals. Intuitively, each page is associated with a bandit arm and each latency with a loss. \par
Before experiments, we have preprocessed the dataset. We observed that the dataset includes too high latencies. Therefore, we have truncated the latencies at $1000$ ms, which can be thought of as timeout for a more realistic real-world setting. We have normalized the truncated latencies into $[0,1]$, and converted them to the gains by subtracting from $1$. Since there are $1361$ latencies for each URL, we set $T=1361$. The other parameters are selected as $\delta=0.01$ and $S=3$. We note that the value of $S$ we chose might not correspond to the actual segment number in the optimum $m$-arm sequence. Nonetheless, it is selected arbitrarily to simulate the practical cases where the value of $S$ is not available a priori.   \par
Using the universities as the bandit arms, we have extracted $100$ different games with $10$
bandit arms. For each game,
we assumed that the agent chose $5$ arms, i.e., $K=10$ and $m=5$. We have repeated each game $20$ times and plotted the time-averaged regrets in Fig. \ref{fig:unik}. Similar to all of
the tests before, we observe that our algorithm achieves a better performance in all time instances. Additionally, we do benchmarks on the number of segments and the number of arms. For the number of
segments benchmark, we have extracted $20$ $10$-arm games, where the actual segment number in the optimum strategy is in the interval $[20i,20i+19]$ for each game indexed by $i\in\{0,1,\cdots,19\}$. We have repeated each game $20$ times and plotted the ensemble distributions in Fig \ref{fig:unis}.  Interestingly, there is no direct correlation between the segment numbers and the average regret values. Furthermore, as can be observed, irrespective of the number of segments, our algorithm outperforms the other algorithms for all number of switches. For the number of bandit arms, we have used the even numbers from $2$ to $20$ and chose the half of the arms in every game.  For each number of arms, we have extracted 20 different games. We have run each game 20 times plotted the ensemble distributions in Fig. \ref{fig:unia}. As expected, the regret values of the algorithms increase with the number of arms. Moreover, the difference in performances becomes more apparent as $m$ increases, which is consistent with our theoretical results.
\section{Concluding Remarks}\label{sec:conclusion}
We studied the adversarial bandit problem with multiple plays, which is a widely used framework to model online shortest path and online advertisement placement problems\cite{KoolOSP, Nakamura2005}. In this context, as the first time in the literature, we have introduced an online algorithm that truly achieves (with minimax optimal regret bounds) the performance of the best multiple-arm selection strategy. Moreover, we achieved this performance with computational complexity only log-linear in the arm number, which is significantly smaller than the computational complexity of the state-of-the-art\cite{Gyorgy2007}. We also improved the best-known high-probability bound for the multi-play setting by $O(\sqrt{m})$, thus, close the gap between high-probability bounds\cite{Neu2016,Gyorgy2007} and the expected regret bounds\cite{Kale2010,Uchiya2010}. We achieved these results by first introducing a MAB-MP with expert advice algorithm that is capable of utilizing the structure of the expert set. Based on this algorithm, we designed an online algorithm that sequentially combines the selections of all possible $m$-arm selection strategies with carefully constructed weights. We show that this algorithm achieves minimax regret bound with respect to the best switching $m$-arm sequence and it can be efficiently implementable with a weight-sharing network applied on the individual arm weights. Through an extensive set of experiments involving synthetic and real data, we demonstrated significant performance gains achieved by our algorithms with respect to the state-of-the-art adversarial MAB-MP algorithms  \cite{Gyorgy2007,Kale2010,Uchiya2010,Neu2016,Auer1995}.
\appendices
\section{}
\label{appa}
\subsection{Dependent Rounding (DepRound)}
\begin{algorithm}[t!]
\algsetup{linenosize=\small}
\small
	\caption{DepRound}\label{alg:DepRound}
	\begin{algorithmic}[1]
		\STATE \textbf{Inputs:} The subset size $m(<K)$, $(p_1,p_2,\cdots,p_K)$ with $\sum_{i=1}^K p_i=m$
		\STATE \textbf{Output:} Subset of $[K]$ with $m$ elements
		\WHILE{there is an $i$ with $0 < p_i < 1$} 
		\STATE Choose distinct $i$ and $j$ with $0 < p_i < 1$ and $0 < p_j < 1$
		\STATE Set $\alpha = \min(1-p_i, p_j)$ and $\beta = \min(p_i,1-p_j)$
		\STATE Update $p_i$ and $p_j$ as 
		\begin{align*}
			(p_i,p_j)=  \left\{
			\begin{array}{ll}
 			(p_i+\alpha,p_j-\alpha) &\textrm{with probability } \frac{\beta}{\alpha+\beta} \\
 			(p_i-\beta,p_j+\beta) &\textrm{with probability } \frac{\alpha}{\alpha+\beta}
			\end{array} \right.
		\end{align*}
		\ENDWHILE
		\RETURN  $\{i : p_i = 1, 1 \geq i \geq K \}$
	\end{algorithmic}
\end{algorithm}
To efficiently select a set of $m$ distinct arms from $[K]$,  we use a nice technique called dependent rounding (DepRound) \cite{Gandhi2006}, see Algorithm \ref{alg:DepRound}. DepRound takes as input the subset size $m$ and the arm probabilities $(p_1, p_2, \cdots, p_K)$ with $\sum_{i=1}^K p_i=m$. It updates the probabilities until all the components are $0$ or $1$ while keeping the
sum of probabilities unchanged, i.e., $m$. The while-loop is executed at most
$K$ times since at least one of $p_i$ and $p_j$ becomes $0$ or $1$ in each time of the
execution. The algorithm updates the probabilities in a randomized manner such that it keeps the expectation values of
$p_i$ the same, namely, $E[p^{t+1}_i ] = E[p^t_i]$ for every $i \in [K]$, where $p^t_i$
denotes $p_i$ after the $t^{th}$
execution of the inside of the while-loop. This follows from
\begin{align}
&(p_i + \alpha) \frac{\beta}{\alpha+\beta} + (p_i - \beta)  \frac{\alpha}{\alpha+\beta} = \nonumber \\
&(p_i - \alpha) \frac{\beta}{\alpha+\beta} + (p_i + \beta)  \frac{\alpha}{\alpha+\beta} = p_i,
\end{align}
which indicates that each arm in the output is selected by its marginal probability $p_i$.
Since the while-loop is executed at most
$K$ times, DepRound
runs in $O(K)$ time and $O(K)$ space.
\subsection{Arm Capping}
\begin{algorithm}[t!]
\algsetup{linenosize=\small}
\small
	\caption{Capping algorithm}
	\begin{algorithmic}[1]
	\label{alg:cap}
		\STATE \textbf{Input:} The subset size $m(<K)$, $(v_1,v_2,\cdots,v_K)$ with $\sum_{j=1}^K v_i=1$ \label{cap:l1}
		\STATE $\bm{v^\downarrow}$  $\leftarrow$ Sort $(v_1,v_2,\cdots,v_K)$ in a descending order \label{cap:l2}
		\STATE $\textrm{\textbf{indices}}^\downarrow$ $\leftarrow$ Keep the original indices of the sorted weights \label{cap:l3}
		\STATE $\textrm{upper\_bound} = \frac{(1/m) -(\gamma/K)}{(1-\gamma)}$ \label{cap:l4}
		\STATE $i \leftarrow 1$	\label{cap:l6}
		\STATE $\textbf{temp} \leftarrow \bm{v^\downarrow}$ \label{cap:l7}
		\REPEAT
		\STATE ( Set first $i$ largest components to $\textrm{upper\_bound}$ and normalize the rest to $(1- i*\textrm{upper\_bound})$ ) \label{cap:l8}
		\STATE $\textbf{temp} \leftarrow \bm{v^\downarrow}$ \label{cap:l9}
		\STATE $\textrm{temp}(j)= \textrm{upper\_bound}$ for $j=1,\cdots,i$ \label{cap:l10}
		\STATE $\textrm{temp}(j)= (1- i*\textrm{upper\_bound} ) \frac{\textrm{temp}(j)}{\sum_{l=i+1}^K \textrm{temp}(l)}$ for $j=i+1,\cdots,K$ \label{cap:l11}
		\STATE $i \leftarrow i+1$ \label{cap:l12}
		\UNTIL{$\max(\textbf{temp}) \leq \textrm{upper\_bound}$} \label{cap:l13}
		\STATE $(v_1,v_2,\cdots,v_K) \leftarrow$ Replace the entries of \textbf{temp} by using $\textrm{\textbf{indices}}^\downarrow$  \label{cap:l14}
		\RETURN 
	\end{algorithmic}
\end{algorithm}
In this section, we describe how we find the threshold $\alpha_t$ and cap the weights, i.e the lines \ref{alg1:cap_init}-\ref{alg1:cap_fin} in Algorithm \ref{alg:algexp4m}, and the lines \ref{alg2:cap_init}-\ref{alg2:cap_fin} in Algorithm \ref{alg:algexp3ms}. The
presented algorithm in Algorithm \ref{alg:cap} simultaneously finds the threshold $\alpha_t$ and caps the weights. \par
In the algorithm, we start with sorting the arm weights in a descending order (line \ref{cap:l2}). Then, we set the largest $i$ arm weights to $\frac{(1/m) -(\gamma/K)}{(1-\gamma)}$ (line \ref{cap:l10}) and normalize the other weights (line \ref{cap:l11}) so that the sum of the
weights stays $1$. By the lines \ref{cap:l10}, \ref{cap:l11}, and \ref{cap:l13}, we aim to satisfy
\begin{align}
\label{cap:eq}
\frac{\alpha_t}{\sum\limits_{v_j(t) \geq \alpha_t} \alpha_t + \sum\limits_{v_j(t) < \alpha_t} v_j(t) }&= \frac{(1/m) -(\gamma/K)}{(1-\gamma)}
\end{align}
subject to $\max(v_j(t)) = \alpha_t$. Since $\frac{(1/m) -(\gamma/K)}{(1-\gamma)} > \frac{1}{m}$, there is always an $i < m$ that satisfies Eq. (\ref{cap:eq}) and it can be found in $O(m)$ step. After finding $i$, the algorithm
replaces the capped arm-weights (line \ref{cap:l14}) and returns. Since the most expensive
operation in the algorithm is sorting, the algorithm requires
$O(K \log K)$ time complexity.
\section{}
\label{appb}
\begin{proof}[\textbf{Proof of Theorem \ref{the4}}]
Define $q_i(t)= w_i(t)/W_t$, where $W_t= \sumNr w_i(t)$, and $\yt=\hat{y}_i(t) + c \vh/\sqrt{KT}$. By following the first steps of the proof of \textit{Exp3.M} (up to inequality (4) in \cite{Uchiya2010}), we can write
{\small
\begin{equation}
\lnratioT \leq \eta \sumT \sumNr q_i(t) \yt + \eta^2 \sumT \sumNr q_i(t) \yt^2
\end{equation}
}with the assumption of $\eta \yt \leq 1$. By using the AM-GM inequality, we get:
{\small
\begin{align}
\lnratioT &\geq \sumrA \frac{\ln(w_r(T+1))}{m} - \ln \frac{W_1}{m} \nonumber \\
&= \frac{\eta}{m} \sumT \sumrA \yrt + \frac{1}{m} \sumrA \ln(w_r(1)) - \ln \frac{W_1}{m}. \label{1}
\end{align}
}where $A^*$ is the set defined in (\ref{astar}). To bound the terms with $\yt$, we give two useful facts:
{\small
\begin{equation}
\sumNr \frac{w_i(t)\zeta^i_j(t)}{W_t}=v_j(t) \leq \frac{v_j'(t)}{\sum_{l=1}^K v_l'(t)} \leq \frac{p_j(t)}{m(1-\gamma)} \textrm{,  } j \in [K]-U_0(t)\label{fact1}
\end{equation}
}where we use {\small $\sum_{l=1}^K v_l'(t) \leq \sum_{l=1}^K v_l(t) = 1$}. For the second fact, let us say {\small $d_i(t)=\sum_{j \in [K] - U_0(t)} \zeta^i_j(t)$}. Then, {\small
\begin{align}
\sumNr \frac{w_i(t)}{W_t} \hat{y}_i(t)^2 &= \sumNr \wfr d_i(t)^2 \Big( \sumTh \frac{\zeta^i_j(t)}{d_i(t)} \xh \Big)^2 \nonumber \\
&\leq \sumNr \wfr \Big( \sumTh \zeta^i_j(t) \xh^2 \Big) \label{evb2} \\
&\leq \frac{1}{m(1-\gamma)}   \sumK \xh \label{evb3}
\end{align}
}where we use {\small $E[X]^2 \leq E[X^2]$} 
and $d_i(t) \leq 1$ in (\ref{evb2}), then 
(\ref{fact1}) and {\small $p_j(t)\xh \leq 1$} in (\ref{evb3}). Next, we bound 
terms with $\yt$:
{\small
\begin{align}
&\sumNr q_i(t) \yt =  \sumNr \wfr \Big( \sumTh \zeta^i_j(t) \xh + \frac{c}{\sqrt{KT}} \frac{\zeta^i_j(t)}{p_j(t)} \Big) \nonumber \\
&= \sumTh \sumNr \frac{w_i(t) \zeta^i_j(t)}{W_t} \Big(\xh + \frac{c}{p_j(t) \sqrt{KT}}  \Big)  \nonumber \\
&\leq \frac{1}{m(1-\gamma)} \Big( \sumTh p_j(t) \xh \Big) + \frac{c}{m(1-\gamma)} \sqrt{\frac{K}{T}}. \label{2}
\end{align}
}{\small
\begin{align}
&\sumNr q_i(t) \yt^2 \leq \sumNr \wfr \Big( \yh + \ub \Big)^2 \nonumber \\
&\leq \frac{2}{m(1-\gamma)} \Big( \sumK \xh \Big) + \frac{2c^2}{KT} \frac{K}{m \gamma} \sumNr \wfr \! \sumTh \! \frac{\zeta^i_j(t)}{p_j(t)} \label{3} \\
& \leq \frac{2}{m(1-\gamma)} \Big( \sumK \xh \Big) + \frac{2c^2K}{T m^2 \gamma (1-\gamma)} \label{5}
\end{align}
}where we use $(a+b)^2 \leq 2(a^2+b^2)$ 
and $\hat{u}_i(t) \leq K/(\gamma m)$ in (\ref{3}). By using (\ref{1}), (\ref{2}) and (\ref{5}) and by noting that $p_j(t)=1$ for $j\in U_0(t)$, we get: 
{\small
\begin{align}
&\frac{\eta}{m} \sumT \sumrA \bm{\zeta}^r(t) \cdot \bm{\hat{x}}(t) + \frac{\eta}{m} \frac{c}{\sqrt{KT}} \sumT \sumrA \hat{v}_r(t)  - \ln \! \frac{W_1}{m} \nonumber \\
&+ \frac{1}{m} \! \sumrA \ln(w_r(1)) \! \leq \! \frac{\eta}{m(1-\gamma)} G_{Exp4.MP}  + \frac{\eta c \sqrt{KT}}{m(1-\gamma)}    \nonumber \\
& + \frac{2 \eta^2 c^2 K}{\gamma m^2 (1-\gamma)}+\frac{2 \eta^2}{m(1-\gamma)} \sumT \sumK \xh
\end{align}
}where {\small $\bm{\zeta}^r(t) \cdot \bm{\hat{x}}(t)= \sumK \zeta^r_j(t) \xh$}. By dividing both sides with {\small $\eta/(m(1-\gamma))$}, and noting that {\small $ \sumT \sumK \xh \leq (K/m) \hat{\Gamma}_{A^*}$}, the statement in the theorem can be obtained.
\end{proof} %
\begin{proof}[\textbf{Proof of Corollary \ref{probcor}}]
The proof consists of two steps. First, we prove an auxiliary result to help us to derive high-probability bound. Second, by using the auxiliary result and Theorem \ref{the4} we prove the statement in the corollary. In the first step, we use the beautiful martingale property given in Theorem 1 in \cite{Beyg2011}. Let us say, {\small$\yit=y_i(t)-\yh$} for any fixed $i \in [N_r]$, where {\small $y_i(t)=\sumTh \zeta^i_j(t)x_j(t)$} and {\small$\yh= \sumTh \zeta^i_j(t) \xh$}. We point out that
\begin{equation*}
E[\yit]=0, \qquad \yit \leq 1, \qquad E[\yit ^2]\leq \hat{u}_i(t).
\end{equation*}Let us define
{\small \begin{equation}
\label{vv}
V' \define \frac{KT}{m} \normalsize{\textrm{ and }} 
\sigma_i \define \sqrt{\frac{m}{KT}} \sumT \vh + \sqrt{\frac{KT}{m}}.
\end{equation}}With the assumption of $\ln (N_r/\delta) \leq (e-2) KT/m$, by Theorem 1 in \cite{Beyg2011}, we can write
{\small \begin{equation}
\textbf{Pr}\Big[\sumT \yit \geq \sqrt{(e-2)\ln \frac{N_r}{\delta}}\sigma_i\Big] \leq \frac{\delta}{N_r}
\end{equation}}for any $i \in [N_r]$. By applying union of events over the set $[N_r]$, and noting $(e-2)<1$, we get
{\small \begin{equation}
\label{event1}
\textbf{Pr}\Big[\forall i \in [N_r] : \sumT \yit \leq \sqrt{\ln \frac{N_r}{\delta}}\sigma_i\Big] \geq 1 - \delta
\end{equation}}Since the event in (\ref{event1}) includes every $i \in [N_r]$, we can sum any $m$ of them without changing the 
bound. Then we get  
{\small
\begin{equation}
\textbf{Pr}\Big[\forall A \in \textbf{C}([N_r],m) : \sum_{i \in A} \sumT \yit \leq \sqrt{\ln\frac{N_r}{\delta}}\sum_{i \in A}  \sigma_i\Big] \geq 1 - \delta.
\end{equation}
}Note that {\small$\sumT \yit = \sumT \sumTh \zeta^i_j(t)(x_j(t)-\xh)$}. Since $x_j(t)=\xh$ for $j \in U_0(t)$, we can equivalently write
{\small
\begin{equation}
\label{aux}
\textbf{Pr}\Big[\forall A \in \textbf{C}([N_r],m) : \sum_{i \in A} G_i - \hat{G}_i \leq \sqrt{\ln\frac{N_r}{\delta}}\sum_{i \in A}  \sigma_i\Big] \geq 1 - \delta.
\end{equation}}Lastly, we point out that, according to \eqref{vv},
{\small
$$\sqrt{\ln\frac{N_r}{\delta}}\sum_{i \in A}  \sigma_i= \sqrt{m \ln\frac{N_r}{\delta}} \Big(\frac{1}{\sqrt{KT}} \sumT \sum_{i \in A}  \vh + \sqrt{KT}\Big).$$}\par
In the second step, we first observe that our parameter selection satisfies the assumption (2) in Theorem \ref{the4}. Then, we restate the result of the theorem when $w_i(1)=1$ $\forall i \in [N_r]$, and $\eta= m\gamma/(2K)$:
{\small
\begin{equation}
(1-2\gamma) \Gamma_{A^*} - (1- \gamma) \frac{2K}{\gamma} \ln \frac{N_r}{m} \leq G_{Exp4.MP} + c\sqrt{KT}+c^2.
\end{equation}
}By using (\ref{aux}), 
$\Gamma_{A^*}$ from (\ref{astar}), and noting {\small $c=\sqrt{m \ln(N_r/\delta)}$}, we get {\small $\textbf{Pr}[G_{max} \leq \Gamma_{A^*} + c\sqrt{KT}] \geq 1 - \delta$}. Then,
\begin{equation}
G_{max} - G_{Exp4.MP} \leq \frac{2K}{\gamma} \ln \frac{N_r}{m} + 2 \gamma G_{max} + 2 c \sqrt{KT} + c^2
\end{equation}
holds with probability at least $1-\delta$. We point out that since $\gamma$ has a small value, we ignore $(1- \gamma)$ and $(1-2\gamma)$ terms at the right hand side. In the end, by noting that $G_{max} \leq mT$ and using the given parameters in the corollary, the statement can be obtained.
\end{proof}%

\section{}
\label{appc}
Before the proof, we give one technical lemma:
\begin{lemma} \label{l1}
For any $S>1$ and $T > S$,
{\small
$$e^{1-S} \leq (1- \frac{S-1}{T-1})^{T-S}.$$
}\end{lemma}
\begin{proof}
By taking the natural logarithms of both sides, we get $S-1 \geq \ln ( 1+ \frac{S-1}{T-S})^{T-S}$. The fact that $\ln(1+\frac{\alpha}{x})^x \leq \alpha$ for $x\geq 0$ completes the proof .
\end{proof}
\begin{proof}[\textbf{Proof of Theorem \ref{switchcor}}]
In this proof, we again begin with proving an auxiliary statement to derive high-probability regret bound. Fix $\sT$ and say $\xst=x_{\sT(t)}-\hat{x}_{\sT(t)}$. Then, 
\begin{equation*}
E[\xst]=0, \qquad \xst \leq 1, \qquad E[\xst ^2]\leq 
1/p_{\sT(t)}.
\end{equation*}
We define
{\small
\begin{align}
&\Delta=\frac{\delta'}{K} \Big(\frac{\beta}{K-1} \Big)^{\Sw-1} (1-\beta)^{T-\Sw} \\
&\Delta'=\frac{\delta'}{K} \Big(\frac{\beta}{K} \Big)^{\Sw-1} (1-\beta)^{T-\Sw}
\end{align} 
}where $\delta'\in[0,1]$ and {\small$\beta=\betav$}. We note that since $S$ is the same for all the elements of the set $Z$, $\Delta$ and $\Delta'$ are arbitrary constants in $[0,1]$. We use the same $V'$ in (\ref{vv}) and define  
{\small \begin{align*}
\sigma_{\sT} \define \!\sqrt{\frac{m}{KT}} \sumT \! \frac{1}{p_{\sT(t)}} +  \sqrt{\frac{KT}{m}}.
\end{align*}}With the assumption of $\ln (1/\Delta') \leq (e-2) KT/m$, by Theorem 1 in \cite{Beyg2011}, we can write
{\small
\begin{equation}
\textbf{Pr}\bigg[\sumT  \xst \geq \sqrt{(e-2) \ln \frac{1}{\Delta'}} \sigma_{\sT} \bigg] \leq \Delta.
\end{equation}
}Applying a union bound over 
$Z$ and noting 
$\sum_{Z} \Delta \leq \delta$, 
{\small
\begin{equation}
\textbf{Pr}\bigg[\forall \sT\! \in \! Z\! : \! \sumT \!  \xst \! \leq \! \sqrt{\ln \frac{1}{\Delta'}}  \sigma_{\sT} \bigg] \! \geq \! 1-\delta. \label{i2}
\end{equation}
}In order to get a clear expression, we aim to write $\ln \frac{1}{\Delta'}$ in terms of $\delta$. Therefore, we aim to satisfy
\begin{equation}
\label{satistfy}
\delta^{\Sw-1} \leq \beta^{\Sw-1} (1-\beta)^{T-\Sw}.
\end{equation} 
By Lemma \ref{l1}, {\small$\delta' \leq \betave$} satisfies inequality (\ref{satistfy}). Then, by writing
{\small$\delta'=  \betave \delta$}, where $\delta \in [0,1]$, and summing any $m$ $\sT$ in (\ref{i2}), we get
{\small
\begin{align}
\textbf{Pr}\bigg[\forall A\! \in \! C(Z,m) \! : \! \sum_{i \in A} \sumT \!  \xst \! \leq \!  c \sigma_A \bigg] \! \geq \! 1- \betave \delta'. \label{i3}
\end{align}
}where 
{\small \begin{equation}
\sigma_A=\Big(\frac{1}{\sqrt{KT}} \sumT \sum_{i \in A}  \frac{1}{p_{\sT(t)}} + \sqrt{KT}\Big)  \nonumber
\end{equation}}and $c$ is one of the given parameters in the theorem. \par
In the second step, we first observe that our parameter selection satisfies the assumption (2) in Theorem \ref{the4}. Second, we introduce a new notation $\{\textbf{s}^1_T,\cdots,\textbf{s}^m_T\}\in\textbf{M}_T$, which means that 
$\MT$ can be written as the combination of single arm sequences $\textbf{s}^1_T,\cdots,\textbf{s}^m_T$. We point out that the single arm sequences  $\{\textbf{s}^1_T,\cdots,\textbf{s}^m_T \}\in \mb$ can be selected the ones with the same switching instants, and the same segment number. Then, their prior weights become 
{\small
\begin{equation}
\label{m2}
w_{\sT}(1)= \frac{1}{K} \Big( \frac{\beta}{K-1} \Big)^{\Sw-1} (1-\beta)^{T-\Sw} 
\end{equation}
}for $\{\textbf{s}^1_T,\cdots,\textbf{s}^m_T \}\in \mb$. To have a bound 
w.r.t. $\mb$, we
define
{\small
\begin{equation}
\label{m1}
\hat{\Gamma}_{\mb}  \define  \hat{G}_{\mb}+  \sumT \summb  \vst. 
\end{equation}
}We note that $\hat{\Gamma}_{\mb}  \leq \hat{\Gamma}_{A^*}$. We also point out that $W_1=1$ by our prior scheme.  
Then by using $\eta=\etav$, (\ref{m1}), and (\ref{m2}) in inequality (\ref{regretformula}), we can write
{\small
\begin{align}
&(1-2 \gamma) \Big(\hat{G}_{\mb}+ \sumT \summb \vst + c \sqrt{KT} \Big) - 2c\sqrt{KT} \nonumber \\
& - \frac{2K}{\gamma} \ln \frac{K^{\Sw}}{ \beta^{\Sw-1} (1-\beta)^{T-\Sw}}  + c^2 \leq G_{Exp4.MP}
\end{align}
}where we use $\gamma <0.5$, $w_{\sT}(1) \leq 1$, and
{\small
\begin{equation}
w_{\sT}(1)\geq \frac{1}{K} \Big( \frac{\beta}{K} \Big)^{\Sw-1} (1-\beta)^{T-\Sw} \textrm{ for } \{\textbf{s}^1_T,\cdots,\textbf{s}^m_T \}\in \mb.
\end{equation} }By (\ref{i3}) and Lemma \ref{l1}, if we use the given $c$ and $\beta$ values, 
{\small
\begin{align}
\label{m3}
&(1-2\gamma) G_{\mb} -G_{Exp4.MP} \leq  \frac{2K \Sw}{\gamma} \ln \frac{eK(T-1)}{\Sw-1} \nonumber \\
&+ 2 \sqrt{mK \Sw T \lnsd} + m \Sw \lnsd
\end{align}
}holds with probability at least {\small$1-\betave \delta$}. By noting $G_{\mb}\leq mT$ and using the given $\gamma$ value, the statement in the theorem can be obtained. 
\end{proof}

\begin{proof}[\textbf{Proof of Theorem \ref{theoem51ext}}]
We use $\nst$, $\Nstm$ defined in Section \ref{sec:exp3ms}, and
{\small \begin{equation}
\nj= (\xh+ \frac{c}{p_j(t) \sqrt{KT}})\1_{j \not \in U_0(t)} \textrm{ for } j \in [K].
\end{equation}}Let $\wst$ be the weight of an arbitrary sequence $\st$ , $r$ be an arbitrary arm, and $\overline{v}_r(t)$ be the weight of the arm $r$ in the hypothetical \textit{Exp4.MP} run at round $t$, given by 
{\small \begin{equation}
\overline{v}_r(t)= \sumsrt \frac{\wst}{\sumst \wst}.
\end{equation}}In the proof, we use mathematical induction to show $\overline{v}_r(t)=v_r(t)$ for any $r \in [K]$ and $t=1,2,\cdots,T$. The proof begins with noting $v_j(1)=w_{\textbf{s}_1}=1/K$. Then,
{\small
\begin{align}
\overline{v}_r(t)&= \sumsrt \frac{\wst}{\sumst \wst} = \sumsrt \frac{\pst \exp(\eta \Nstm)}{\sumst \wst} \nonumber  \\
&= \sumsrt \frac{\psst \pstm  \exp(\eta \Nstm)}{\sumst \wst} \nonumber  \\
&= \sumstm \frac{\wstm \exp(\eta \nstm) \psst}{\sumst \wst} \nonumber  \\
&\textrm{Assuming $\sumsjtm \wstm \propto v_j(t-1)$ for $j \in [K]$} \label{assumption}\\
&= \sumK \frac{v_j(t-1) \exp(\eta \njm ) \psst}{\sum_{l \in [K]} v_l(t-1) \exp(\eta \nlm )} \nonumber  \\
&= \sumK \frac{\tilde{v}_j(t-1) \Big(\frac{\beta}{K-1}\1_{j \neq r} + (1-\beta)\1_{j=r}\Big)}{\sum_{l \in [K]} \tilde{v}(t-1)} = v_r(t). \label{result}
\end{align}
}Since our assumption in (\ref{assumption}) holds for $t=1$, by (\ref{result}) it holds for all $t$. Then, the theorem holds for all $t$ as well. 
\end{proof}

\bibliographystyle{IEEEtran}
\balance
\bibliography{my_references}

\vfill

\end{document}